\documentclass{article}

\usepackage[main, final]{neurips_2026}

\usepackage[utf8]{inputenc} % allow utf-8 input
\usepackage[T1]{fontenc}    % use 8-bit T1 fonts
\usepackage{hyperref}       % hyperlinks
\usepackage{url}            % simple URL typesetting
\usepackage{booktabs}       % professional-quality tables
\usepackage{amsfonts}       % blackboard math symbols
\usepackage{nicefrac}       % compact symbols for 1/2, etc.
\usepackage{microtype}      % microtypography
\usepackage{xcolor}         % colors

\usepackage{microtype}
\usepackage{graphicx}
\usepackage{booktabs}
\usepackage{amsmath,amssymb,amsthm}
\usepackage{mathtools}
\usepackage{bbm}
\usepackage{enumitem}
\usepackage{algorithm}
\usepackage{algpseudocode}
\usepackage{hyperref}
\usepackage{bm}
\usepackage{microtype}
\usepackage{graphicx}
\usepackage{subcaption}
\usepackage{booktabs} % for professional tables
\usepackage{tikz}
\usetikzlibrary{arrows.meta}
\usepackage{tikz}
\usetikzlibrary{positioning,arrows.meta}
\usepackage{dirtytalk}

\newcommand{\R}{\mathbb{R}}
\newcommand{\E}{\mathbb{E}}

\newcommand{\eps}{\varepsilon}
\newcommand{\iid}{\stackrel{\mathrm{i.i.d.}}{\sim}}

\theoremstyle{plain}
\newtheorem{theorem}{Theorem}[section]
\newtheorem{proposition}[theorem]{Proposition}
\newtheorem{lemma}[theorem]{Lemma}
\newtheorem{corollary}[theorem]{Corollary}
\theoremstyle{definition}
\newtheorem{definition}[theorem]{Definition}

\theoremstyle{remark}
\newtheorem{remark}[theorem]{Remark}
\title{MaxSketch: Robust Distinct Counting in Streams via Random Projections}

\author{
\begin{tabular}{l}
Nikos Tsikouras$^{1,2}$ \quad Constantine Caramanis$^{2,3}$ \quad Christos Tzamos$^{1,2}$
\vspace{-9pt}\\
\\
\vspace{-1pt}
\normalfont
$^{1}$ National and Kapodistrian University of Athens, Greece \\
\vspace{-1pt}
\normalfont
$^{2}$ Archimedes, Athena Research Center, Greece \\
\vspace{-1pt}
\normalfont
$^{3}$ The University of Texas at Austin, USA \\
\end{tabular}
}

\begin{document}

\maketitle

\begin{abstract}
    Estimating the number of distinct elements in a data stream is well understood when repeated elements are identical. In modern settings, however, observations are high-dimensional and noisy, so repeated instances of the same object are only approximately similar -- for example, different images of the same individual may vary significantly at the pixel level. Classical sketches such as HyperLogLog rely on consistent hash values for identical elements and break down in this regime. Recent work on robust distinct counting in general metric spaces gives algorithms that use $\tilde O(\sqrt n/\varepsilon)$ words of memory and proves an $\Omega(\sqrt n)$-bit lower bound for
    general noisy datasets. We show that substantially improved memory guarantees are possible under geometric structure common in learned representations. We introduce MaxSketch, a simple max-linear sketch built from random Gaussian projections, and prove that it succeeds in estimating the number of distinct latent objects. Concretely, we show that under this geometric structure $m = O( \varepsilon^{-2}\log (2/\delta))$ independent Gaussian projections suffice to estimate the true distinct count within a factor of $1+\varepsilon$ with probability at least $1-\delta$. Experiments on image streams confirm that MaxSketch accurately estimates distinct counts and generalizes beyond the training regime. Our results bridge classical streaming algorithms and modern representation learning, showing how geometric structure can fundamentally reduce the complexity of distinct counting.
\end{abstract}

% ============================
\section{Introduction}
% ============================

Counting the number of distinct elements in a data stream is a classical algorithmic problem with deep connections to randomized algorithms, information theory, and large-scale data analysis. In its canonical form, the task is to estimate the number of unique keys in a stream using sublinear memory and a single pass over the data. A long line of work has shown that this is possible with striking efficiency: algorithms such as \cite{durand2003loglog,flajolet2007hyperloglog,flajolet1985probabilistic} achieve accurate estimates using only logarithmic memory by maintaining carefully designed randomized sketches of the stream.

These classical results, however, rely on a crucial assumption: the same element is represented identically every time it appears. In modern data modalities, this assumption is fundamentally violated. To see this, consider the task of counting the number of distinct people passing through a busy intersection observed by a camera. Suppose Alice and Bob each pass through once, and later Alice passes through again. Although it is the same underlying individual, her second appearance may differ substantially at the pixel level, due to changes in lighting, viewpoint, or pose. As a result, the two observations of Alice may be nearly as different from each other as they are from observations of entirely different individuals.

% A natural approach is to embed each observation into a continuous feature space and apply distinct counting in that representation. However, this does not resolve the issue: as we have mentioned already repeated observations of the same underlying object are no longer identical, causing sketching methods to overcount, while clustering-based approaches require maintaining a number of representatives that grows with the number of distinct elements, undermining the sublinear memory guarantees of streaming algorithms. This raises a central question:

A natural approach is to embed each photo into a continuous feature space and apply distinct counting in that representation. However, this does not resolve the issue: repeated observations of the same underlying object are no longer identical, causing sketching methods to overcount, while clustering-based approaches require maintaining a number of representatives that grows with the number of distinct elements, undermining the sublinear memory guarantees of streaming algorithms.

A recent line of work \cite{chen2016streaming, chen2018distinct, zhang2025robust} has studied this \textit{robust} distinct counting problem, where elements are considered equal if they are sufficiently close in some metric. While these results give algorithms with non-trivial guarantees, they also reveal that the problem is fundamentally harder than its classical counterpart: \cite{zhang2025robust} establishes that, on general noisy datasets, any streaming algorithm that approximates the distinct count within a factor of $1.1$ requires $\Omega(\sqrt n)$ bits of memory, exponentially far from the rates achievable under exact key equality \cite{kane2010optimal}.

Modern representation learning, however, often induces a more useful geometric structure than the distance-based separation studied above: observations of the same object form tight clusters in embedding space, while different objects correspond to well-separated directions.

This raises a central question:

\emph{Closeness of representations in an arbitrary metric space is not strong
enough to allow robust distinct counting with memory independent of $n$ (${\Omega}(\sqrt{n})$ lower bound in \cite{zhang2025robust}). Does clustered geometry with intra-class closeness and inter-class separation bypass this barrier?}

This question sits at the intersection of classical streaming algorithms and modern representation learning. On the one hand, the theory of distinct counting demonstrates that extreme-value statistics of random hash functions contain sufficient information about the number of unique elements \cite{flajolet2007hyperloglog,kane2010optimal}. On the other hand, permutation-invariant architectures such as DeepSets \cite{zaheer2017deep} can in principle aggregate set-structured data, but provide no theoretical guarantees on recovering global combinatorial properties like distinct counts. In this work, we bridge this gap with a deliberately minimal max-linear sketch that is both analytically tractable and provably correct.

\subsection{A geometric model for distinct objects}

% We assume that each distinct object corresponds to an unknown latent direction in a high-dimensional space, and that each observation is a noisy but well-aligned version of its object's direction. Different objects occupy well-separated directions, with low inner products between distinct centers. This is not an abstraction: contrastive and classification objectives explicitly encourage this configuration, and the neural collapse phenomenon formalizes why it emerges in the terminal phase of training even without explicit geometric supervision \cite{papyan2020prevalence, zhou2022optimization}.

We assume that each distinct object corresponds to an unknown latent direction in a high-dimensional space, and that each observation is a noisy but well-aligned version of its object’s direction. Different objects occupy well-separated directions, with low inner products between distinct centers. This picture is motivated by neural collapse: at the end of training, deep classifiers map the training examples of each class close to their class mean,
and the class means become nearly orthogonal \cite{papyan2020prevalence, zhou2022optimization}. Face-recognition
embeddings are trained to give the same structure for identities not seen
during training \cite{schroff2015facenet}.

Under this model, the number of distinct objects corresponds to the number of well-separated directions generating the stream. The challenge is to recover this number without explicitly clustering and without storing the embeddings.

The $\Omega(\sqrt n)$ lower bound of \cite{zhang2025robust} uses adversarial datasets
that are not \say{well-shaped}: points of different objects can be as close as
near-duplicates of the same object.  Our
clusterability assumption (Definition \ref{def:clusterable}) is not chosen to sidestep this bound; it is motivated by the
geometry of learned representations. Clusterable streams fall in the well-shaped class, for which \cite{zhang2025robust} explicitly leaves open whether $\Omega(\sqrt n)$ memory is necessary. We show that it is not necessary under the stronger inner-product structure of clusterability. This structure makes Gaussian extreme-value statistics separate different cardinalities, and $O(\varepsilon^{-2}\log(2/\delta))$ Gaussian projections suffice.

% This angular structure places us outside the worst-case general-metric setting of \cite{zhang2025robust}. In particular, \cite{zhang2025robust} proves an $\Omega(\sqrt n)$ lower bound for robust distinct counting in general metric spaces, while explicitly leaving open whether such a lower bound holds even for well-shaped datasets. Our assumptions are stronger still: we require a Euclidean embedding in which observations are tightly aligned with latent centers and distinct centers have low inner products. This inner-product structure enables Gaussian extreme-value statistics whose expectations separate different cardinalities, yielding polylogarithmic memory rather than the $\tilde{\Omega}(\sqrt n)$ worst-case behavior in the general-metric setting.

\subsection{MaxSketch: a max-linear analogue of classical sketches}
Our key observation is that max aggregation discards multiplicity by construction: once a cluster has contributed a large projection value, additional repetitions of the same object leave the sketch unchanged. This contrasts sharply with sum or mean aggregation, which scale with frequency and therefore confuse distinctness with multiplicity. Given random Gaussian vectors $w_1, \dots, w_m \in \mathbb{R}^d$, we compute

\[M_j(X) = \max_{1\le i\le n} \langle w_j, x_i\rangle\]

where $x_1, \dots, x_n \in \mathbb{R}^d$ are observations generated from $k$ underlying objects,
and average these maxima across projections. The expectation of this statistic grows monotonically with the number of distinct latent directions, which is the signal we want to recover. This operation is permutation-invariant, streaming-compatible, and stores  only $m$ scalars per stream.

% This operation is permutation-invariant, streaming-compatible, and requires storing only 
% $m$ scalars.

We call this representation MaxSketch. Unlike classical sketches such as HyperLogLog, MaxSketch operates directly on continuous embeddings and tolerates noise and approximate equality. Unlike general-purpose neural set models, MaxSketch is deliberately minimal and analytically tractable.

\subsection{Theoretical guarantees}

Our main result shows that MaxSketch provably recovers the number of distinct objects up to multiplicative error.

\textbf{Informal Theorem \ref{thm:main} (Robust Distinct Count Estimation).}
Under a natural clusterability assumption, $m = O( \varepsilon^{-2}\log (2/\delta))$ random projections suffice to estimate the number of distinct objects in a stream of length $n$ within a $(1+\varepsilon)$-multiplicative error, with probability at least $1-\delta$. 

Compared to the $\Omega(\sqrt{n})$ memory required in general metric spaces~\cite{zhang2025robust}, the improvement is enabled by separation between cluster centers, a property missing from the worst-case
instances underlying that bound.%

Our analysis shows that, despite correlations induced by approximate cluster structure, the expected projection maximum is close to the expected maximum of independent Gaussians. We prove that this expectation rises by a detectable amount when the distinct count grows by a factor of $1+\varepsilon$, and use superconcentration to control fluctuations around that expectation. These ingredients yield the stated projection bound.

This is a key distinction with other methods. We view MaxSketch as a statistics-preserving embedding: rather than preserving distances or geometry, it maps a set of observations to a single number, $S(X)$, whose expectation grows with the number of distinct objects. This
has two consequences. When the assumptions hold only approximately, $S(X)$ remains informative, as illustrated empirically in the face counting experiments of Section~\ref{subsection:counting faces}. And once the
projections are fixed, learning to count reduces to fitting a one-dimensional
monotone function of $S(X)$, which needs only polynomially many labelled streams
(Theorem~\ref{thm:learn}).

\subsection{End-to-end learning from count-only supervision}

The theoretical guarantees are stated for fixed Gaussian projections, but MaxSketch slots naturally into an end-to-end pipeline: an encoder maps raw inputs to embeddings, MaxSketch aggregates them, and a learned readout produces the count. Our theory does not assume that projections or embeddings are learned optimally; it establishes only that a compact sketch \emph{exists} whenever the embeddings satisfy the geometric assumption. We examine this empirically in Section~\ref{sec:experiments}: on MNIST,
end-to-end training under count-only supervision induces well-clustered
embeddings; on CIFAR-10, where intra-class variability is much higher,
MaxSketch counts almost exactly on top of a frozen self-supervised encoder.

\subsection{Contributions}
\label{sec:contributions}
\begin{itemize}
    \item \textbf{MaxSketch}, a simple max-linear, permutation-invariant sketch
    inspired by classical extreme-value methods. We prove that under a clusterability condition,
    $m=O(\varepsilon^{-2}\log(2/\delta))$ Gaussian projections give a $(1+\varepsilon)$-multiplicative approximation with
    probability at least $1-\delta$. By
    contrast, the best known algorithm for well-shaped data uses memory growing roughly like $\sqrt{n}$, and on general noisy datasets $\Omega(\sqrt n)$ bits are necessary
    \cite{zhang2025robust}.

    \item \textbf{A learning reduction} showing that, once projections are
    fixed, count prediction reduces to fitting a one-dimensional
    monotone function of a fixed scalar statistic.

    \item \textbf{Empirical validation} on MNIST, CIFAR-10, and CelebA face
    streams, demonstrating extrapolation to sequences up to ten times longer than
    those seen in training, and accurate counting of thousands of unseen identities on faces.
\end{itemize}

\subsection{Related Work}

\paragraph{Distinct counting and sketching.}
Estimating the number of distinct elements in a data stream is a classical problem in algorithms, with celebrated solutions such as Flajolet–Martin, LogLog, and HyperLogLog achieving near-optimal accuracy using logarithmic memory via randomized hashing and extreme-value statistics \cite{durand2003loglog,flajolet2007hyperloglog,flajolet1985probabilistic,kane2010optimal}. These methods crucially rely on exact key equality: repeated occurrences of the same element must hash to the same value. A natural idea in continuous settings is to first apply locality-sensitive hashing (LSH) \cite{indyk1998approximate, indyk1997locality} and then run HyperLogLog on the resulting hash codes. However, this approach fundamentally fails for distinct counting: LSH guarantees collision for some near neighbors, not all repetitions of the same object, and its collision probability depends on distance in a way that entangles object identity with multiplicity. As a result, LSH-based preprocessing introduces bias that cannot be corrected by classical sketches.

\paragraph{Robust distinct counting under near-duplicates.}

A line of work initiated by \cite{chen2016streaming,chen2018distinct} studies robust distinct counting under near-duplicates,
including well-shaped datasets in low-dimensional Euclidean spaces and, more recently, general
metric spaces. Zhang \cite{zhang2025robust} gives algorithms for general metric spaces that
use $O(\varepsilon^{-1}\sqrt{n\log(1/\varepsilon)})$ words on
well-shaped datasets and $\tilde O(\sqrt n/\varepsilon)$ words on general
noisy datasets. He also proves that
$\Omega(\sqrt n)$ bits are necessary for a $1.1$-approximation on general
noisy datasets, and leaves open whether this holds for well-shaped data.
We compare with these bounds in the word model of \cite{zhang2025robust}, in which one word stores one element of the space, here a vector in $\mathbb{R}^d$; we likewise count a real number as one word. In this model
MaxSketch uses $O(m)$ words: the $m$ real numbers and the $m$ projection vectors, which are drawn once, before the stream, and do not depend on it.

MaxSketch addresses a different, more structured regime. The stream is embedded in a Euclidean
space, observations are tightly aligned with latent centers, and distinct centers have small inner
products. This inner-product structure is stronger than a metric distance-separation condition:
it allows Gaussian projection maxima to behave as extreme-value statistics whose expectations
separate different cardinalities. Under this clusterability assumption, MaxSketch uses
$O( \varepsilon^{-2}\log (2/\delta))$ Gaussian projections, showing that substantially smaller sketches are possible
in structured learned representations than in the worst-case general-metric setting.

\section{Problem Setting}

We observe a stream of unit vectors \(
X = ( \tilde{x}_1, \dots, \tilde{x}_n ) \subset \mathbb{S}^{d-1},\) arriving in arbitrary order. Each $\tilde{x}_i$ is a noisy observation of one of $k^\star$ unknown latent objects; the same object may appear multiple times, with each appearance producing a different vector. Our goal is to estimate $k^\star$ in a single pass using memory sublinear in $n$, with supervision (when available) limited to the total count.

This setting differs from classical distinct-elements counting in that exact equality between repeated occurrences is replaced by approximate geometric similarity. To make the problem well-posed, we assume the stream has the following structure.

\begin{definition}[$(\eta,\rho)$-clusterable sequence]\label{def:clusterable}
Fix $\eta \in (0,1)$ and $\rho \in [0,1)$. A sequence
\(X = ( \tilde{x}_1, \dots, \tilde{x}_n )\) of unit vectors in $\mathbb{R}^d$ is \emph{($\eta,\rho$)-clusterable} with $k^\star$ centers if there exist distinct unit vectors \(
x^{(1)}, \dots, x^{(k^\star)} \in \mathbb{R}^d
\) and a surjective assignment $r : [n] \to [k^\star]$ such that:
\begin{enumerate}[leftmargin=1.2em, itemsep=0.2em]
    \item (\textbf{Within-cluster alignment}): 
    \(
    \langle \tilde{x}_i, x^{(r(i))} \rangle \geq 1 - \eta/2
    \quad \text{for all } i.
    \)
    
    \item (\textbf{Cross-cluster separation}): 
    \(
    |\langle x^{(r)}, x^{(s)} \rangle| \leq \rho
    \quad \text{for all } r \neq s.
    \)
\end{enumerate}
\end{definition}

Throughout, $\tilde x_i$ denotes an observed point and $x^{(r)}$ a cluster center; when no clusterability is assumed (e.g. in Definition \ref{def:maxsketch}) we write a generic stream as $x_1,\dots,x_n$.

The first condition says each observation is directionally close to its center (capturing noise, augmentation, or pose variation); the second says distinct centers are well-separated, with $\rho$ controlling approximate independence under Gaussian projection. In particular, when $\eta$ is small, clusterable streams
are well-shaped in the sense of \cite{zhang2025robust}.\footnote{Specifically, when $\eta < (1-\rho)/18$, a condition much
weaker than the one our main result assumes. Take $\alpha=2\sqrt\eta$ in the definition of
\cite{zhang2025robust}. Since
$\|\tilde x_i-x^{(r(i))}\|^2=2-2\langle\tilde x_i,x^{(r(i))}\rangle\le\eta$,
observations of the same object are within distance $\alpha$ of each other,
while observations of different objects are at least
$\sqrt{2-2\rho}-2\sqrt\eta>2\alpha$ apart.} This assumption formalizes a structure widely observed in learned representations. When training embeddings for classification, or other self-supervised objectives, the learned space typically exhibits clustering by semantic identity: embeddings of the same object (or class) cluster tightly, while distinct objects occupy well-separated directions. Recent work on neural collapse \cite{han2021neural,papyan2020prevalence, zhou2022optimization, zhu2021geometric} formalizes this phenomenon in deep networks.  Our results do not require learning this geometry, only that it exists.

\section{MaxSketch: A Max-Linear Streaming Representation}

We now describe a simple sketch that exploits this geometry.

\begin{definition}[\textbf{MaxSketch}]\label{def:maxsketch}
Fix an integer $m$ and let $w_1, \dots, w_m \sim N(0,I_d)$ be independent random vectors, fixed once and for all. Given a stream $X = ( x_1, \dots x_n )$, define

\[
M_j(X) = \max_{1\le i\le n} \langle w_j, x_i\rangle, \hspace{0.3cm}
S(X) = \frac{1}{m}\sum_{j=1}^m M_j(X).
\]

\paragraph{Sketch size.} As is common in the sketching literature, we
measure the size of MaxSketch by its dimension $m$, the number of Gaussian
projections. The sketch maps a stream $X$ to the vector
$(M_1(X),\dots,M_m(X))\in\mathbb R^m$ and updates it in a single pass.
The
projection vectors are drawn once and reused for every stream; only the
$m$ values $M_j$ depend on the stream.

\paragraph{Why max aggregation?}

Max aggregation discards multiplicity information by construction: once a cluster has contributed a large projection value, additional repetitions do not change the sketch. This sharply contrasts with sum or mean aggregation, which scale with frequency and therefore confuse distinctness with multiplicity. As we show, under clusterability, maxima of random projections isolate extreme-value statistics that depend primarily on the number of distinct latent directions.
\end{definition}

\section{Theoretical Guarantees}\label{sec:theory}
We show that MaxSketch estimates the number of distinct clusters under $(\eta, \rho)$-clusterability. The analysis has three steps: (i) characterize the extreme-value behavior of Gaussian projections onto well-separated centers, (ii) bound the perturbation introduced by within-cluster noise, and (iii) combine an expectation gap with concentration to distinguish different cluster counts using $O\left(\varepsilon^{-2} \log(2/\delta)\right)$ projections. 

\subsection{Gaussian maxima over cluster centers}

Fix unit vectors $x^{(1)}, \dots, x^{(k)} \in \mathbb{R}^d$ and $w \sim \mathcal{N}(0, I_d)$, and let \(G_r := \langle w, x^{(r)} \rangle.\)
Then $G = (G_1, \dots, G_k)$ is jointly Gaussian with mean zero, unit variances, and
\(\operatorname{Cov}(G_r, G_s) = \langle x^{(r)}, x^{(s)} \rangle.\) We control $\mathbb{E}[\max_r G_r]$ when the centers are well-separated by comparing to the i.i.d.\ case, where extreme-value asymptotics are well understood.

\begin{lemma}[Gaussian maxima comparison]\label{lem:slepiann-full}
Let $(G_1,\dots,G_k)$ be a centered Gaussian vector with $\mathrm{Var}(G_r)=1$
and $|\mathrm{Cov}(G_r,G_s)|\le\rho$ for $r\neq s$, where $\rho\in[0,1)$.
Let $Z_1,\dots,Z_k\stackrel{\text{i.i.d.}}{\sim}\mathcal N(0,1)$. Then
\[
\sqrt{1-\rho}\;\mathbb{E}\Big[\max_{r\le k} Z_r\Big]
\;\le\;
\mathbb{E}\Big[\max_{r\le k} G_r\Big]
\;\le\;
\sqrt{1+\rho}\;\mathbb{E}\Big[\max_{r\le k} Z_r\Big].
\]
\end{lemma}

\subsection{Robustness to within-cluster perturbations}

Maxima over observed points differ from maxima over centers by an amount controlled by the within-cluster slack $\eta$.

\begin{lemma}[Perturbation bound]\label{lem:perturbb}
Let $\widetilde X=(\tilde x_1,\dots,\tilde x_n)\subset\mathbb S^{d-1}$ satisfy the within-cluster alignment condition of Definition~\ref{def:clusterable} with centers $x^{(1)},\dots,x^{(k)}$ and assignment $r$. Let $w\sim\mathcal N(0,I_d)$. Then
\[
\Big|\,\mathbb{E}\!\left[\max_{i \leq n} \langle w, \tilde{x}_i \rangle\right]
-
\mathbb{E}\!\left[\max_{r \leq k} \langle w, x^{(r)} \rangle\right]\Big|
\;\leq\; \sqrt{2\eta \log (2n)}.
\]
\end{lemma}

\subsection{Separating different numbers of clusters}

To estimate the number of clusters, we must distinguish between $k$ and $(1+\varepsilon)k$ centers via their expected maxima.

\begin{lemma}[Gap in Gaussian maxima]\label{lem:gap-iid}
Let \(k\ge2\), \(\varepsilon\in(0,1]\), and \(k'=\lceil(1+\varepsilon)k\rceil\). Let $x^{(1)}, \dots, x^{(k')} \in \mathbb{R}^d$ be unit vectors with $|\langle x^{(i)}, x^{(j)} \rangle| \leq \rho$ for $i \ne j$, where $\rho \le c_\rho \varepsilon / \log k$ for a sufficiently small absolute
constant $c_\rho>0$, and let $w \sim \mathcal{N}(0, I_d)$. Define
\[
\Delta_{\mathrm{centers}}(k) := \mathbb{E}\!\left[\max_{1 \leq i \leq k'} \langle w, x^{(i)} \rangle\right]
- \mathbb{E}\!\left[\max_{1 \leq i \leq k} \langle w, x^{(i)} \rangle\right].
\]
Then $\Delta_{\mathrm{centers}}(k) = \Omega(\varepsilon / \sqrt{\log k})$ for all $k \geq 2$.
\end{lemma}

% To estimate the true $k$ from $S(\tilde{X})$, we must ensure that between-cluster 
% correlations and within-cluster noise do not obscure the signal, and that a $(1+\varepsilon)$ change in cluster count produces a 
% gap large enough to detect. Lemmas~\ref{lem:slepiann-full}--\ref{lem:gap-iid} 
% establish exactly these three facts.
% Lemmas~\ref{lem:slepiann-full}--\ref{lem:gap-iid} are proved in Appendix~\ref{app:prelims}. 

Each lemma isolates one ingredient: Lemma~\ref{lem:slepiann-full} controls
the effect of correlation between centers through $\rho$, Lemma~\ref{lem:perturbb} controls the within-cluster noise through $\eta$, and Lemma~\ref{lem:gap-iid} gives us the signal that
distinguishes $k$ from $(1+\varepsilon)k$ centers. Theorem~\ref{thm:main}
combines them with a concentration bound for $S(\widetilde X)$. All three
lemmas are proved in Appendix~\ref{app:prelims}.

\subsection{Main counting theorem}

\begin{theorem}[Estimating the number of clusters]
\label{thm:main}
Fix $n\ge2$, $\varepsilon\in(0,1/2]$, and $\delta\in(0,1)$.
Let
\(
\widetilde X
=
(\widetilde x_1,\ldots,\widetilde x_n)
\subseteq\mathbb S^{d-1}
\)
be $(\eta,\rho)$-clusterable with $2\le k^\star\le n$ centers, where
\(
\rho
\le
c_\rho\frac{\varepsilon}{\log n},
\eta
\le
c_\eta\frac{\varepsilon^2}{(\log n)^2},
\)
for sufficiently small absolute constants $c_\rho,c_\eta>0$.

Draw
\(
w_1,\ldots,w_m
\stackrel{\mathrm{i.i.d.}}{\sim}
\mathcal N(0,I_d),
m\ge
C\frac{\log(2/\delta)}{\varepsilon^2},
\)
for a sufficiently large absolute constant $C$.

Then there exists an explicit nondecreasing estimator, that uses only the MaxSketch statistic of Definition~\ref{def:maxsketch},
\[
S(\widetilde X)
=
\frac1m\sum_{j=1}^m
\max_{1\le i\le n}
\langle w_j,\widetilde x_i\rangle,
\]
that outputs $\widehat k$ satisfying
\[
k^\star
\le
\widehat k
\le
(1+\varepsilon)k^\star
\]
with probability at least $1-\delta$ over the Gaussian projections.
\end{theorem}

\begin{remark}[Estimator]\label{rem:estimator}
The estimator uses a geometric grid of
$O(\log n/\varepsilon)$ candidate output values. Associated with
each grid point is a deterministic threshold
$\theta_0\le\theta_1\le\cdots$.
Given the observed value $S(\widetilde X)$, the estimator returns
the output associated with the first threshold satisfying
\(
S(\widetilde X)\le \theta_r.
\)
If no threshold satisfies this inequality, it returns $n$.
Thus the estimator is a deterministic nondecreasing function of
the scalar statistic $S(\widetilde X)$. The thresholds depend only on $n$ and $\varepsilon$.
\end{remark}

Two features of Theorem~\ref{thm:main} are worth noting. First, the number
of projections depends only on $\varepsilon$ and $\delta$: it does not grow
with the stream length $n$, the number of objects $k^\star$, or the dimension
$d$. The stream length enters only through the geometric conditions, which
become stricter as $n$ grows. Second, the estimator reads a single number,
$S(\widetilde X)$, and is a fixed nondecreasing function of it
(Remark~\ref{rem:estimator}). The proof is in Appendix~\ref{app:thm1}.

Can a better use of the registers reduce the number of projections? The next
proposition shows that, as far as the dependence on $\varepsilon$ and
$\delta$ is concerned, it cannot.

\begin{proposition}[Gaussian-max lower bound]
\label{prop:gaussian-max-lower}
Fix $\varepsilon\in(0,1/4]$ and $\delta\in(0,1/8]$. Let $k^\star$ be an
unknown integer. Suppose an estimator observes only $m$ independent samples
from the distribution of $M_{k^\star}=\max_{1\le r\le k^\star}Z_r$, where
$Z_r\stackrel{\mathrm{i.i.d.}}{\sim}\mathcal N(0,1)$, and must output
$\widehat k$ with $k^\star\le\widehat k\le(1+\varepsilon)k^\star$ with
probability at least $1-\delta$ for every $k^\star\ge\lceil1/\varepsilon\rceil$.
Then $m=\Omega(\varepsilon^{-2}\log(1/\delta))$.
\end{proposition}

Consider a noiseless stream whose $k^\star$ centers are orthogonal. It
satisfies the conditions of Theorem~\ref{thm:main}, and each MaxSketch
register is the maximum of $k^\star$ independent standard Gaussians.
Proposition~\ref{prop:gaussian-max-lower} therefore says that no estimator
working from the register values alone, ours included, can do better than
Theorem~\ref{thm:main} in its dependence on $\varepsilon$ and $\delta$,
except by constant factors. The proposition says nothing about estimators
that also see the projection vectors $w_j$, or about other streaming
algorithms. The proof is in Appendix~\ref{app:gaussian-max-lower}.

\subsection{From random to fixed weights, and learning the readout}

Theorem~\ref{thm:main} provides guarantees over a fresh draw of projections for every clusterable input. In practice,
however, the sketch is built once and deployed across many sequences
from some unknown distribution $\mathcal D$. Corollary~\ref{cor:fixed-sketch} shows that a single fixed sketch suffices for most inputs from $\mathcal{D}$.

\begin{corollary}[Fixed sketch for a distribution]
\label{cor:fixed-sketch}
Under the conditions of Theorem~\ref{thm:main}, let $\mathcal D$ be any distribution over $(\eta,\rho)$-clusterable sequences of length at most $n$ with at least two centers. Then there exist fixed $w_1, \dots, w_m$ such that MaxSketch achieves a $(1 + \varepsilon)$-multiplicative approximation with probability at least $1-\delta$ over a fresh draw $X\sim \mathcal{D}$.
\end{corollary}

Once the projections \(W=(w_1,\ldots,w_m)\) are fixed, each sequence
is represented by the scalar statistic \(S_W(X)\). Theorem~\ref{thm:main}
constructs a deterministic threshold readout whose thresholds do not
depend on \(X\) or on the realized \(W\). For each fixed clusterable
sequence \(X\) of length \(n\), the readout returns a
\((1+\varepsilon)\)-approximation with probability at least \(1-\delta\)
over the initial draw of \(W\). Corollary~\ref{cor:finite-threshold-readout} (proved in
Appendix~\ref{app:cor2}) extends this pointwise guarantee to sequences of length
at most \(n\) with at least two centers, and states explicitly that
the readout has only \(O(\log n/\varepsilon)\) possible outputs. We next
show that, for sequences drawn from an unknown distribution, a
monotone readout can be learned efficiently from count-labeled
examples.

% Once the projections are fixed, each sequence is represented by the scalar statistic $S(X)$. Theorem~\ref{thm:main} constructs an explicit threshold readout that, with high probability over the Gaussian sketch, returns a \((1+\eps)\)-approximation for every clusterable sequence. Corollary~\ref{cor:finite-threshold-readout} (proved in Appendix~\ref{app:cor2}) shows that this readout can be replaced by an equivalent finite-valued monotone readout with only $O(\log n/\eps)$ possible outputs. We next show that, when sequences are drawn from an unknown distribution, such a monotone readout can be learned efficiently from count-labeled examples.

\begin{theorem}[Learnability]
\label{thm:learn}
Under the conditions of Theorem~\ref{thm:main}, let $\mathcal D$ be any
distribution over $(\eta,\rho)$-clusterable sequences of length at most $n$ with at least two centers.
There exists a polynomial-time learning algorithm which, given
\(
N
=
\operatorname{poly}(\log n, 1/\varepsilon, 1/\delta)
\)
i.i.d.\ training samples
\(
(X_1,k_1),\ldots,(X_N,k_N)\sim\mathcal D,
\)
outputs a monotone readout function
$f:\mathbb R\rightarrow\mathbb R$
such that, for an independent test sample
$(X,k)\sim\mathcal D$,
\(
k
\le
f(S_W (X))
\le
(1+\varepsilon)k
\)
with probability at least $1-\delta$ over the sketch $W$, the training sample and the test sample.
\end{theorem}

Proofs of Corollary~\ref{cor:fixed-sketch} and Theorem~\ref{thm:learn} are in Appendices~\ref{app:cor2}--\ref{app:thm2}.

Together, Theorem~\ref{thm:main} and Theorem~\ref{thm:learn} give a complete pipeline: a fixed Gaussian sketch (good with high probability per Theorem~\ref{thm:main}) followed by a monotone readout (efficiently fittable per Theorem~\ref{thm:learn}). The next section examines whether this signal is useful on real embeddings, whether learned end-to-end from count supervision or supplied by a pretrained encoder.

% ============================
\section{Experiments}\label{sec:experiments}
% ============================

Our experiments have three goals. First, on MNIST we ask whether an encoder and a MaxSketch readout can be trained jointly from count-only supervision, and whether the resulting model extrapolates to sequence lengths far longer than those seen during training -- consistent with the theory once embeddings are clusterable. 

Second, on CIFAR-10, where intra-class variability is substantially higher,
we ask whether MaxSketch still counts exactly when the embedding comes from a
frozen, label-free self-supervised encoder, and whether a single readout
works across stream lengths, including lengths never seen in training. 

Third, we evaluate MaxSketch on real face streams using fixed pretrained embeddings, where the assumptions of our theory are only approximately satisfied, and we compare against the streaming algorithm of Zhang \cite{zhang2025robust}.

\subsection{End-to-end learning to count on MNIST}

We study the problem of learning to count the number of distinct latent objects in a sequence under aggregate supervision only. Each data point is a sequence $X = (x_1,\dots,x_n)$ of MNIST images, and the only supervision is the number of distinct latent classes $k(X) = \vert\{z_1,\dots,z_n\}\vert$, appearing in $X$; the image labels $z_i \in \{1,\dots,10\}$ are never observed. The model consists of an encoder $\phi_{\theta}: \mathcal{X} \rightarrow \mathbb{R}^d$, applied independently to each image, and an aggregator $f_\psi: (\mathbb{R}^d)^n \rightarrow \mathbb{R}$ applied to the resulting set of embeddings, both trained jointly to predict $k(X)$.

Beyond MaxSketch, we include two natural baselines that share the same encoder but differ in $f_\psi$: a DeepSets aggregator \cite{zaheer2017deep}, which sum-pools embeddings, and an LSTM aggregator \cite{hochreiter1997long} that processes them sequentially. DeepSets is permutation-invariant but conflates frequency with distinctness; the LSTM is included as a non-set baseline.

All three models reach near-perfect accuracy across sequence lengths $n \in \{2, 5, 10, 20, 50, 100\}$ (Figure \ref{fig:mnist_count_generalization}), and $\pm 1$ accuracy (within one of the correct value) is exactly 100\% throughout. This is consistent with the theory: MNIST embeddings have low intra-class variability and strong inter-class separation, so learning an encoder to satisfy the $(\eta,\rho)$-clusterability assumption becomes substantially easier. In this regime, the aggregation mechanism plays a relatively minor role, and even sum- or sequence-based models can recover the correct count. These results are a sanity check rather than a comparison: when the embeddings cluster this well, any reasonable aggregator can count.

The more informative test is whether the learned model extrapolates beyond the count and length range seen at training. Table \ref{tab:generalization} reports MaxSketch trained at $n_{\text{train}} \in \{10, 20, 50\}$ and evaluated at $n_{\text{test}} \in \{150, 250, 500\}$. Models trained on shorter sequences exhibit some underestimation bias for very long sequences, but keep high $\pm 1$ accuracy, indicating that the learned statistic still carries meaningful count information.

The learned encoder, evaluated independently by 1-NN classification on test images, attains $1.00$ accuracy regardless of $n_{\text{train}}$, so the failure at long sequences is not a representation collapse. Rather, it is a calibration failure of the regression head. This behaviour matches the theoretical picture: under clusterability the MaxSketch statistic
tracks $k$, so what remains is calibrating the readout. A readout fitted on short streams is miscalibrated on much longer ones. In a stream of $10$ images, not all of the chosen classes appear, so the readout rarely sees large counts during training.

\begin{table}[!hbt]
\caption{Generalization of MaxSketch models to unseen, longer sequence lengths on MNIST. Each model is trained on a specific sequence length $n_\text{train}$ and evaluated on sequences of lengths $n_\text{test}$ not seen during training. Numbers indicate exact count accuracy, with $\pm 1$ accuracy shown in parentheses.}
\label{tab:generalization}
\centering
\begin{tabular}{lccc}
\toprule
$n_\text{train}$ & $n_\text{test}=150$ & $n_\text{test}=250$ & $n_\text{test}=500$ \\
\midrule
10  & $0.37$ $(0.97)$ & $0.25$ $(0.90)$ & $0.06$ $(0.78)$ \\
20  & $0.70$ $(0.99)$ & $0.58$ $(0.98)$ & $0.37$ $(0.96)$ \\
50  & $0.99$ $(1.0)$ & $0.98$ $(1.0)$ & $0.96$ $(1.0)$ \\
\bottomrule
\end{tabular}
\end{table}

\subsection{Counting on CIFAR-10 with a frozen encoder}
\label{sec:cifar}

Self-supervised encoders are trained without labels to map views of the same
object close together and different objects far apart, which is close to the
geometry our analysis assumes. We therefore test MaxSketch on CIFAR-10 with a
frozen DINOv2 encoder~\cite{oquab2023dinov2}. CIFAR-10 class labels are used
only to build streams and their count targets, exactly as in the MNIST
experiment. Full details can be found in
Appendix~\ref{app:cifar10}.

We use two readouts, both fitted with count supervision only: an MLP on the
$m$ registers, trained once on streams of lengths $n\le 100$ and evaluated
without retraining on streams of length up to $n = 500$, and the one-dimensional
isotonic map of $S(X)$ analysed in Section~\ref{sec:theory}, calibrated
separately for each $n$.

Table~\ref{tab:cifar} reports the results over three random seeds. With
DINOv2-L features, the MLP readout recovers the exact count on 91--100\% of
test streams, and every prediction is within one of the truth at every
length. Its accuracy does not degrade at $n=250$ and $n=500$, which are longer
than any training stream; the smaller DINOv2-B encoder is only slightly
behind. The isotonic readout of the scalar $S(X)$ is within one on 94--100\%
of streams and exact on 53--71\%. This is expected: with up to ten classes, exact counting means telling $k$ apart from $k+1$, a relative difference of as small as 10\%, and Theorem~\ref{thm:main} requires more projections and tighter clusters the finer the distinction. The MLP is not covered by our analysis, but it reads all $m$ registers instead of their average, and so has more information to work with.

As in the MNIST experiment, the readout is trained and tested on the same ten classes, with disjoint training and test images.
 
\begin{table}[t]
\centering
\caption{Counting on CIFAR-10 with a frozen self-supervised encoder (DINOv2, trained without labels). Embeddings are centered and projected onto their top principal directions, fitted on unlabeled training images. Each readout is trained with count supervision on streams of training images and evaluated on streams of test images, with $k\in\{1,\dots,10\}$ and $m=4{,}096$. The MLP readout is trained once on lengths $n\le 100$, so $n=250$ and $n=500$ are lengths not seen during training; the isotonic readout of $S(X)$ is calibrated separately for each $n$. Numbers indicate exact count accuracy (mean over 3 seeds; standard deviation $\le 0.03$), with $\pm 1$ accuracy shown in parentheses.}
\label{tab:cifar}
\small
\setlength{\tabcolsep}{4pt}
\begin{tabular}{lcccccc}
\toprule
Encoder, readout & $n=10$ & $n=20$ & $n=50$ & $n=100$ & $n=250$ & $n=500$ \\
\midrule
DINOv2-L, MLP on registers
  & \textbf{0.91} (1.00) & \textbf{0.91} (1.00) & \textbf{0.97} (1.00)
  & \textbf{1.00} (1.00) & \textbf{1.00} (1.00) & \textbf{1.00} (1.00) \\
DINOv2-B, MLP on registers
  & 0.88 (1.00) & 0.89 (1.00) & 0.96 (1.00)
  & 0.98 (1.00) & 0.98 (1.00) & 0.96 (1.00) \\
DINOv2-L, isotonic on $S(X)$
  & 0.54 (0.96) & 0.53 (0.94) & 0.65 (0.98)
  & 0.71 (1.00) & 0.70 (1.00) & 0.66 (0.99) \\
\bottomrule
\end{tabular}
\end{table}

\subsection{Counting Unique Faces in Real-World Image Sequences}\label{subsection:counting faces}

We now turn to the setting of counting people in face streams. Here the number of distinct objects runs into the thousands, and the embedding comes from a face-recognition encoder that is never adapted to counting. We use the CelebA face dataset
\cite{liu2015deep}, restricted to identities with at least two images so
that each identity can appear multiple times in a stream. We also remove
identities whose mean embeddings are nearly identical (cosine similarity
above $0.85$), which likely denote the same person under two labels. The
remaining $9{,}868$ identities are split uniformly at random into a
calibration half and an evaluation half of $4{,}934$ identities each. For each image we detect and crop the face with MTCNN
\cite{zhang2016joint} and compute a $512$-dimensional embedding with a
pretrained FaceNet encoder \cite{schroff2015facenet}, trained on VGGFace2
\cite{cao2018vggface2}, which is disjoint from CelebA. Embeddings are
centred using the mean of the calibration identities, L2-normalized, and
held fixed throughout, so the encoder is never adapted to the counting
task.

% To simulate heterogeneous data sources, we merged the three datasets into a single pool of identities. The resulting identity set was randomly split into two equal parts: $50\%$ for calibration and $50\%$ for evaluation, yielding 5{,}103 identities in each split. For each image, we extracted a face embedding using the pretrained DeepFace VGG-Face model \cite{serengil2024lightface}. These embeddings were treated as fixed feature vectors and were not further fine-tuned for the counting task.

Each evaluation stream is built by drawing a target distinct count $k \in \{100, 200, \dots, 4900\}$ from the held-out identities and then sampling $n = 20{,}000$ images of those identities, in which the same person typically appears many times. This is the regime motivating our setting in the first place: a video feed or a large photo collection in which repeated observations are the rule. We use $m=4{,}096$ Gaussian projections and a
one-dimensional monotone readout fit by isotonic regression on streams
built from calibration identities; the readout is fit once and applied
unchanged to every evaluation stream. Because $k$ ranges into the
thousands we report the accuracy within $\pm5\%$ and $\pm10\%$ of $k$.

The most directly comparable prior work is the streaming algorithm of Zhang for distinct count estimation in general metric spaces. The algorithm is parameterized by a near-duplicate distance threshold $\alpha$ (Definition 1.1 of \cite{zhang2025robust}), and its correctness guarantee assumes the dataset is \textit{well-shaped}: every two points belonging to different underlying objects must be at distance strictly greater than $2\alpha$, while every two points belonging to the same object must be at distance at most $\alpha$. Real face embeddings do not satisfy this condition at any $\alpha$ as intra-identity and inter-identity distance distributions overlap. One should therefore expect the algorithm to be sensitive to the choice of $\alpha$ and to errors of both signs. We implement Algorithm 1 of \cite{zhang2025robust} and choose
$\alpha\in\{0.6,0.7,\dots,1.4\}$ to minimize its relative error on
calibration streams (smaller thresholds split nearly every identity;
at $\alpha=0.4$ the estimate is about $7.5$ times the true $k$), which gives $\alpha=1.1$. With this single threshold,
Zhang's estimate is off by $61\%$ on average: it is accurate only around
$k\approx500$ and saturates near $780$ for larger $k$, since with many
identities almost every new face lies within $\alpha$ of a stored representative.

By contrast, due to MaxSketch’s nature as a \say{statistical embedding} mild violations of $(\eta,\rho)$-clusterability shift the expectation of $S(X)$ continuously, and the calibrated readout absorbs these biases. Over all evaluation streams (Figure~\ref{fig:MaxSketch}), MaxSketch has a mean
relative error of $4.7\%$, and $91\%$ of its estimates lie within $\pm10\%$
of $k$. Accuracy decreases slowly with $k$: within $\pm10\%$, it is $99\%$
for $k\le1{,}000$, $94\%$ for $1{,}000<k<2{,}500$, and $85\%$ for
$k\ge2{,}500$. Within the stricter $\pm5\%$ band, around $60\%$ of estimates are correct overall, falling from around $80\%$ for $k \le 1{,}000$ to around $50\%$ for $k \ge 2{,}500$. Overall, this demonstrates that
MaxSketch can extract multiplicity information directly from pretrained
embedding spaces without requiring clustering or identity supervision at
inference time. Such approximate counts may be useful in large-scale
monitoring and analytics scenarios where rough population estimates are
sufficient and explicit identity resolution is unnecessary.

\begin{figure}[!hbt]
    \centering
    \includegraphics[width=1\linewidth]{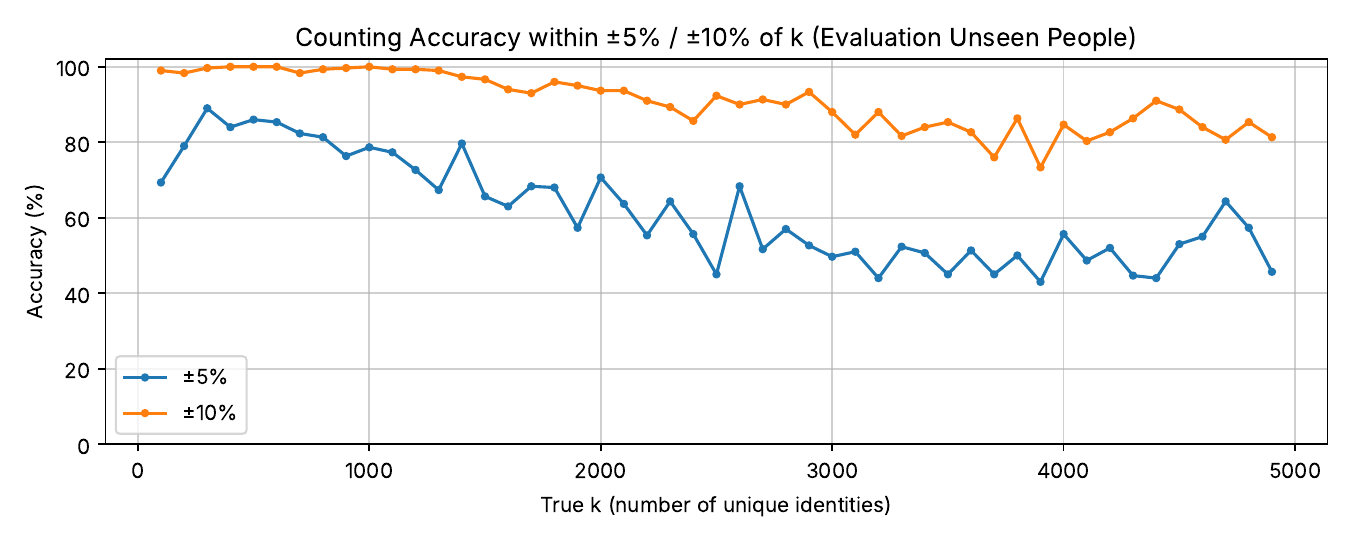}
    \caption{MaxSketch ($m=4096$) on CelebA face streams of $n=20{,}000$ images
from $k$ unseen identities: fraction of streams whose estimate is within
$\pm5\%$ and $\pm10\%$ of $k$, averaged over three independent sketches,
for $k\in\{100,200,\dots,4900\}$.}

    \label{fig:MaxSketch}
\end{figure}

\section{Limitations}\label{sec:limitations}
Our theoretical guarantees rely on an explicit geometric separation
assumption: embeddings must be $(\eta,\rho)$-clusterable, with small
within-cluster variance and sufficiently weak correlations across cluster
centers. When this assumption is violated, e.g., in visual data with
substantial intra-class variability or semantically overlapping
categories, our guarantees no longer apply: within-cluster spread perturbs
the projection maxima, and the gap between different cluster counts is no
longer guaranteed. The face experiments illustrate this regime. MaxSketch
still counts within about $5\%$ on average, but its error grows with $k$.

Additionally, understanding when count supervision alone can provably induce clusterable embedding geometry, rather than merely exploit it when present, remains a central open problem.

\section{Conclusion}

We introduced MaxSketch, a max-linear streaming sketch built from random Gaussian projections, and showed that it provably estimates the number of distinct latent objects in a stream under a natural clusterability assumption. The central insight is that geometric structure common in modern learned representations, that is tight within-cluster alignment and weak cross-cluster correlation, is precisely the structure needed to make extreme-value statistics informative about cardinality and avoid the known $\sqrt{n}$ memory cost that is necessary for general noisy data. Under this assumption, MaxSketch achieves $(1 + \varepsilon)$-multiplicative accuracy using only $O( \varepsilon^{-2}\log (2/\delta))$ Gaussian projections.

On the practical side, MaxSketch counts almost exactly when embeddings are
well clustered: on MNIST with an end-to-end encoder, and on CIFAR-10 with a
frozen label-free DINOv2 encoder, where a single MLP readout extrapolates to
streams five times longer than any seen in training. On real face streams,
where clusterability holds only approximately, the calibrated readout absorbs
residual bias from mild violations, yielding robust approximate counts across
a wide identity range.

% On the practical side, MaxSketch achieves high accuracy when embeddings are well-clustered, as in MNIST or pretrained CIFAR features, but its exact accuracy decreases on datasets with substantial intra-class variability, such as end-to-end CIFAR. On real face streams, where clusterability holds only approximately, the calibrated readout absorbs residual bias from mild violations, yielding counts within about 5\% on average across a wide identity range.

Two open problems would sharpen the theoretical picture.  First, Proposition~\ref{prop:gaussian-max-lower} shows that MaxSketch's projection count is optimal up to absolute constants for estimators based only on independent Gaussian-max values. A lower bound for general streaming algorithms under our clusterability assumption remains open. Second, our guarantees assume clusterability is already present in the encoder; a theory of when aggregate count supervision alone can provably induce this geometry would bring the theoretical and empirical settings into alignment, replacing a geometric assumption on the encoder with a condition on the learning objective alone.

\newpage

\section*{Acknowledgments}
This work was done in Archimedes/Athena RC supported by a PhD fellowship.

Constantine Caramanis was partially supported by the NSF IFML Institute (NSF 2019844), and the
NSF AI-EDGE Institute (NSF 2112471).

\section*{Competing interests} The authors declare no competing interests.

\bibliography{references}
\bibliographystyle{plain}

\newpage

\appendix

% ============================================================
% APPENDIX: COMPLETE PROOFS FOR THEOREMS 1--3
% Drop this into your paper after \appendix.
% ============================================================

\section{Appendix: Extended Related Work}\label{app:related}

\subsection{Classical Sketching and Cardinality Estimation}
The problem of estimating the number of distinct elements in a stream has been extensively studied in the data management and streaming algorithms literature. Flajolet and Martin \cite{flajolet1985probabilistic} pioneered probabilistic counting, using
$O(\varepsilon^{-2}\log n)$ bits for a $(1+\varepsilon)$-approximation.
LogLog \cite{durand2003loglog} and HyperLogLog \cite{flajolet2007hyperloglog} store $O(\varepsilon^{-2})$
counters of $O(\log\log n)$ bits each, for $O(\varepsilon^{-2}\log\log n)$
bits in total. Kane, Nelson and Woodruff \cite{kane2010optimal} give an
algorithm with the optimal $O(\varepsilon^{-2}+\log n)$ bits.

These classical methods share two defining features: (i) they assume deterministic, identical representations for repeated elements (via random hashing), and (ii) they use hand-designed random projections and analytical estimators derived from probabilistic analysis. While highly effective, this paradigm raises a natural question: can modern machine learning discover task-specific sketching statistics end-to-end?

% \subsection{Weak Supervision and Aggregate Labels}
% Weak supervision encompasses label noise \cite{goldberger2016training}, incomplete labels \cite{joachims2003learning}, and crowdsourced labels \cite{raykar2010learning}. Our setting—learning from sequence-level counts without instance-level identity labels—is a specific instance of weak supervision but with distinct structure: the target is a global statistical property (cardinality) rather than a per-instance label.

\subsection{Self-Supervised Learning and Representation Quality}
Our method's success depends critically on good embeddings satisfying clusterability (Definition 2.1). Recent breakthroughs in self-supervised learning have produced embeddings where similar items cluster and dissimilar items separate, directly supporting our geometric assumptions.

\textbf{SimCLR} \cite{chen2020simple}: Uses contrastive loss to learn representations where augmentations of the same image cluster, and different images separate. Achieves clustering structure aligned with our geometric model, producing embeddings with strong within-class alignment and between-class separation.

\textbf{MoCo} \cite{he2020momentum}: Momentum contrast provides an alternative contrastive framework, showing robust clustering across diverse datasets without requiring large batch sizes. Demonstrates the robustness of clustering assumptions across different SSL architectures.

\textbf{BYOL} \cite{grill2020bootstrap}: Learns clustering without explicit negative pairs, showing that clustering emerges naturally even without contrastive objectives. Further validates that learned representations naturally satisfy clusterability.

\textbf{Neural Collapse Phenomenon} \cite{papyan2020prevalence,zhou2022optimization}: Recent work formalizes why deep classifiers learn near-orthogonal class representations in the terminal phase of training. \cite{zhu2021geometric} provide further geometric characterization, confirming that learned embeddings have the low-correlation structure we assume.

\subsection{Counting Objects in Images}

\cite{gao2020cnn} surveys methods that regress the number of objects in a single image (e.g., people in a crowd). These methods count every instance in one image, so the same person seen twice is counted twice. They do not estimate the number of distinct objects across a stream of images, and they come with no theoretical guarantees.

\subsection{Permutation-invariant models and set functions}
Permutation-invariant neural networks aggregate unordered inputs via sum, mean, or max pooling. DeepSets \cite{zaheer2017deep} showed that, for countable input domains, every permutation-invariant function on finite sets can be written as $g\big(\sum_i \phi(x_i)\big)$, and this framework has become foundational for learning on variable-sized sets. For continuous inputs the picture is more delicate: fixed-dimensional sum-based representations have fundamental expressivity limitations, and exact decompositions can require latent dimension that grows with the set size \cite{wagstaff2019limitations, soelch2019deep, pellegrini2020learning}. Learnable aggregators have been proposed to address this, but approximating complex set functions often requires going beyond fixed sum or max pooling and can be hard to learn \cite{kortvelesy2023generalised, pellegrini2020learning}. The choice of aggregator also determines which statistics can be recovered: sum and mean pooling conflate frequency with distinctness. In contrast, we analyze a specific max-linear sketch and give finite-sample, multiplicative error guarantees under an explicit geometric clusterability assumption, explaining why maxima of random projections isolate extreme-value statistics that grow monotonically with the number of distinct objects.

\section{Proofs}\label{app:main-proofs}

\subsection{Preliminaries and technical lemmas}\label{app:prelims}

\paragraph{Notation.}
As in Definition~\ref{def:maxsketch}, for a sequence $\tilde X=(\tilde x_1,\dots,\tilde x_n)\subset\mathbb{S}^{d-1}$
and projections $w_1,\dots,w_m\in\mathbb{R}^d$ we write
\[
M_j(\tilde X):=\max_{1\le i\le n}\langle w_j,\tilde x_i\rangle,
\qquad
S(\tilde X):=\frac1m\sum_{j=1}^m M_j(\tilde X).
\]

% ------------------------------------------------------------
% \begin{lemma}[Norm bound for Gaussians]\label{lem:gauss-norm}
% If $w\sim \mathcal{N}(0,I_d)$ then $\E\|w\|\le \sqrt{d}$ and $\Pr(\|w\| \ge \sqrt{d}+t)\le e^{-t^2/2}$.
% \end{lemma}
% \begin{proof}
% The tail bound is standard for $\chi^2_d$ concentration (e.g.\ by Gaussian concentration for the 1-Lipschitz function $w\mapsto \|w\|$).
% Integrating the tail (or Jensen) yields $\E\|w\|\le \sqrt{\E\|w\|^2}=\sqrt{d}$.
% \end{proof}

\begin{lemma}[Explicit form of Lemma \ref{lem:perturbb}]\label{lem:perturb}
Let $\tilde{X} = (\tilde{x}_1, \dots, \tilde{x}_n ) \subset \mathbb{S}^{d-1}$ be
$(\eta, \rho)$-clusterable with centers $x^{(1)}, \dots, x^{(k)} \in \mathbb{S}^{d-1}$
and assignment $r : [n] \to [k]$, so that
$\langle \tilde{x}_i, x^{(r(i))} \rangle \geq 1 - \eta/2$ for all $i \in [n]$.
Let $w \sim \mathcal{N}(0, I_d)$. Then
\[
\Big|\,\mathbb{E}\!\left[\max_{i \leq n} \langle w, \tilde{x}_i \rangle\right]
-
\mathbb{E}\!\left[\max_{r \leq k} \langle w, x^{(r)} \rangle\right]\Big|
\;\leq\; \sqrt{2\eta \log (2n)}.
\]
\end{lemma}

\begin{proof}
For each $i$, define $v_i := \tilde{x}_i - x^{(r(i))}$. By clusterability, $\|v_i\|^2 \le \eta$.  
Then $\langle w, v_i\rangle$ is centered Gaussian with variance $\|v_i\|^2
\le\eta$.

Since $r$ is surjective, $\max_{i}\langle w,x^{(r(i))}\rangle=\max_{r\le k}\langle w,x^{(r)}\rangle$.
Using $|\max_i a_i-\max_i b_i|\le\max_i|a_i-b_i|$,
\[
\Bigl|\mathbb{E}\max_{i}\langle w,\tilde x_i\rangle-\mathbb{E}\max_{r\le k}\langle w,x^{(r)}\rangle\Bigr|
\le \mathbb{E}\max_{i\le n}|\langle w,v_i\rangle| \le \sqrt{2\eta\log(2n)},
\]
where the last step is the standard bound for the maximum of the $2n$ centered
Gaussians $\pm\langle w,v_i\rangle$, each of variance at most $\eta$.

\end{proof}
% ------------------------------------------------------------

\begin{lemma}[Explicit form of Lemma \ref{lem:slepiann-full}]\label{lem:slepian-full}
Let $(G_1,\dots,G_k)$ be a centered Gaussian vector with $\mathrm{Var}(G_r)=1$
and $|\mathrm{Cov}(G_r,G_s)|\le\rho$ for $r\neq s$, where $\rho\in[0,1)$.
Let $Z_1,\dots,Z_k\stackrel{\text{i.i.d.}}{\sim}\mathcal N(0,1)$. Then
\[
\sqrt{1-\rho}\;\mathbb{E}\Big[\max_{r\le k} Z_r\Big]
\;\le\;
\mathbb{E}\Big[\max_{r\le k} G_r\Big]
\;\le\;
\sqrt{1+\rho}\;\mathbb{E}\Big[\max_{r\le k} Z_r\Big].
\]
\end{lemma}

\begin{proof}
We use the Sudakov--Fernique inequality \cite{vershynin2018high}:
if $(X_1,\dots,X_k)$ and $(Y_1,\dots,Y_k)$ are centered Gaussian vectors with
$\mathbb{E}(X_r-X_s)^2\le\mathbb{E}(Y_r-Y_s)^2$ for all $r,s$, then
$\mathbb{E}\max_r X_r\le\mathbb{E}\max_r Y_r$.
For our vector, $\mathbb{E}(G_r-G_s)^2 = 2-2\,\mathrm{Cov}(G_r,G_s)$ for $r\ne s$.

\emph{Lower bound.} Let $Z\sim\mathcal N(0,1)$ be independent of
$Z_1,\dots,Z_k$, and set $G'_r:=\sqrt{1-\rho}\,Z_r+\sqrt{\rho}\,Z$. Then
$\mathbb{E}(G'_r-G'_s)^2 = 2-2\rho \le 2-2\,\mathrm{Cov}(G_r,G_s)
= \mathbb{E}(G_r-G_s)^2$, since $\mathrm{Cov}(G_r,G_s)\le\rho$. By
Sudakov--Fernique, $\mathbb{E}\max_r G_r\ge\mathbb{E}\max_r G'_r$. Since
$\max_r G'_r=\sqrt{1-\rho}\,\max_r Z_r+\sqrt{\rho}\,Z$ and $\mathbb{E}Z=0$,
\[
\mathbb{E}\Big[\max_{r\le k} G_r\Big]\;\ge\;\sqrt{1-\rho}\;\mathbb{E}\Big[\max_{r\le k} Z_r\Big].
\]

\emph{Upper bound.} Set $Y_r:=\sqrt{1+\rho}\,Z_r$. Then
$\mathbb{E}(Y_r-Y_s)^2 = 2+2\rho \ge 2-2\,\mathrm{Cov}(G_r,G_s)
= \mathbb{E}(G_r-G_s)^2$, since $\mathrm{Cov}(G_r,G_s)\ge-\rho$. By
Sudakov--Fernique,
\[
\mathbb{E}\Big[\max_{r\le k} G_r\Big]\;\le\;\mathbb{E}\Big[\max_{r\le k} Y_r\Big]
\;=\;\sqrt{1+\rho}\;\mathbb{E}\Big[\max_{r\le k} Z_r\Big].
\qedhere
\]
\end{proof}

\begin{lemma}[Concentration of Gaussian max projections]\label{lem:conc-full}
Fix unit vectors $u_1,\dots,u_n\in\mathbb{S}^{d-1}$ and let
$f(w)=\max_{i\le n}\langle w,u_i\rangle$. Then $f$ is $1$-Lipschitz, and for
$w\sim\mathcal N(0,I_d)$,
\[
\bigl\|f(w)-\mathbb{E}f(w)\bigr\|_{\psi_2}\le C
\]
for an absolute constant $C$.
\end{lemma}

\begin{proof}
For any $w,w'$,
\[
f(w)-f(w') \le \max_i\langle w-w',u_i\rangle \le \|w-w'\|,
\]
and the same holds with $w,w'$ swapped, so $f$ is $1$-Lipschitz. The bound then
follows from Gaussian concentration for Lipschitz functions
\cite{vershynin2018high}.
\end{proof}

% \begin{lemma}[Concentration of Gaussian max projections]\label{lem:conc-full}
% Fix unit vectors $u_1,\dots,u_n\in \mathbb{S}^{d-1}$ and define $f(w)=\max_{i\le n}\ip{w}{u_i}$.
% Then $f$ is $1$-Lipschitz: $|f(w)-f(w')|\le \|w-w'\|$.
% Consequently, for $w\sim\mathcal{N}(0,I_d)$,
% \[
% \Pr\Big(|f(w)-\E f(w)|\ge t\Big) \;\le\; 2e^{-t^2/2}.
% \]
% Moreover, if $w_1,\dots,w_m\stackrel{i.i.d.}{\sim}\mathcal{N}(0,I_d)$ and $F:=\frac{1}{m}\sum_{j=1}^m f(w_j)$, then
% \[
% \Pr\Big(|F-\E F|\ge t\Big)\;\le\; 2\exp\Big(-\frac{m t^2}{2}\Big).
% \]
% \end{lemma}
% \begin{proof}
% For any $w,w'$,
% \[
% f(w)-f(w') \;=\; \max_i \ip{w}{u_i} - \max_i \ip{w'}{u_i}
% \;\le\; \max_i \ip{w-w'}{u_i} \;\le\; \|w-w'\|,
% \]
% and similarly with $w,w'$ swapped, so $f$ is 1-Lipschitz.
% Gaussian concentration for Lipschitz functions yields the tail.
% For the average, apply independence and subgaussianity: each $f(w_j)-\E f(w_j)$ is $1$-subgaussian, hence the average is $1/\sqrt{m}$-subgaussian.
% \end{proof}

\begin{lemma}[Mean gap]
\label{lem:gap}
Let $Z_1, \ldots, Z_{k'} \iid N(0,1)$ with $k' = \lceil(1+\eps)k\rceil$. Define
\[
\Delta_1(k, \eps) := \E\!\left[\max_{i \le k'} Z_i\right] - \E\!\left[\max_{i \le k} Z_i\right].
\]
There exists an absolute constant $c > 0$ such that for all integers $k \ge 2$ and all $\eps \in (0, 1]$,
\[
\Delta_1(k, \eps) \;\ge\; \frac{c \, \eps}{\sqrt{\log k}}.
\]
\end{lemma}

\begin{proof}
Let
\[
A := \max_{i \le k} Z_i, \qquad B := \max_{k < i \le k'} Z_i, \qquad \ell := k' - k.
\]
Since the index sets are disjoint and the $Z_i$ are independent, $A$ and $B$ are independent. Also $\ell = \lceil(1+\eps)k\rceil - k \ge \eps k$, with $\ell \ge 1$ whenever $\eps k > 0$.

We first show a pointwise inequality. For every realization and every $t \in \R$,
\[
(B - A)_+ \;\ge\; (B - t)_+ \cdot \mathbf{1}_{\{A \le t\}}.
\]

This follows by a simple case analysis:
\begin{itemize}[noitemsep]
\item If $A \le t$ and $B \le t$: LHS $= (B - A)_+ \ge 0$, RHS $= 0$. OK.
\item If $A \le t$ and $B > t$: LHS $= B - A$ (since $B > t \ge A$ implies $B > A$), RHS $= B - t$. We need $B - A \ge B - t$, equivalently $A \le t$, which holds.
\item If $A > t$ and $B \le t$: RHS $= 0$ (indicator vanishes); LHS $\ge 0$. OK.
\item If $A > t$ and $B > t$: RHS $= 0$; LHS $\ge 0$. OK. 
\end{itemize}

Taking expectations and using the independence of $A$ and $B$:
\begin{equation}
\Delta_1(k, \eps) = \E[\max_{i \le k'} Z_i - \max_{i \le k} Z_i]
\stackrel{(*)}{=} \E[(B - A)_+]
\ge \Pr(A \le t) \cdot \E[(B - t)_+].
\label{eq:master}
\end{equation}

The equality $(*)$ uses that $\max_{i \le k'} Z_i = \max(A, B)$, and $\max(A,B) - A = (B - A)_+$.

Set
\[
t := \Phi^{-1}\!\left(1 - \tfrac{1}{2k}\right), \qquad \text{so} \qquad 1 - \Phi(t) = \tfrac{1}{2k},
\]
where $\Phi(\cdot)$ is the CDF of the standard normal distribution.

We will lower bound $\mathbb{E}[(B-t)_+]$ and combine with \eqref{eq:master}.

First, we prove that $\Pr(A \le t) \ge \tfrac{1}{2}$ for all $k \ge 1$.

\[
\Pr(A \le t) = \Phi(t)^k = \left(1 - \tfrac{1}{2k}\right)^k.
\]
Since $(1 - \frac{1}{2k})^k$ is increasing in $k \ge 1$, we have 
\begin{equation}\label{eq:probability}
\Pr(A \le t) = \left(1 - \frac{1}{2k}\right)^k \ge \frac{1}{2}.
\end{equation}

Additionally, we need to prove an upper bound on $t$, for all $k \ge 2$.

% , $t \le \sqrt{2 \log(2k)} \le 2\sqrt{\log k}$, and consequently $t + 2 \le 5\sqrt{\log k}$.

The standard Gaussian tail bound (Chernoff) gives $1 - \Phi(t) \le e^{-t^2/2}$ for $t \ge 0$. Thus $1/(2k) = 1 - \Phi(t) \le e^{-t^2/2}$, giving $t^2/2 \le \log(2k)$, i.e., $t \le \sqrt{2 \log(2k)}$. For $k \ge 2$, $2k \le k^2$, so $\log(2k) \le 2 \log k$, giving $t \le 2 \sqrt{\log k}$.

Furthermore, we need to prove a bound on $t+2$. We see that the ratio  $(2\sqrt{\log k}+2)/(5\sqrt{\log k})$ is decreasing in $k$, so the bound 

\begin{equation}\label{eq:t+2}
    t+2 \le 5\sqrt{\log k}
\end{equation}
 holds for all $k\ge 2$.
We now derive a uniform lower bound on $\Pr(B>t+s)$ for small $s$, which will yield a lower bound on $\E[(B-t)_+]$ via integration.

With $t$ as defined and $s_0 := 1/(2(t+2))$, for all $\eps \in (0, 1]$ and $k \ge 2$,
\[
\E[(B - t)_+] \ge \frac{\kappa \, \eps}{2(t + 2)}, \qquad \kappa := \frac{1 - e^{-1}}{2 e^2}.
\]

We now control $\Pr(B > t + s)$ uniformly for $s \in [0, s_0]$ and then integrate.

\textbf{(a) Hazard-rate bound.} The Gaussian hazard rate is $h(u) := \phi(u)/(1 - \Phi(u))$ for $u \in \R$. We claim:
\begin{equation}
h(u) \le 2(u + 1) \quad \text{for all } u \ge 0.
\label{eq:hazard}
\end{equation}

For $u \ge 1$: by Mills' ratio, $1 - \Phi(u) \ge \frac{u}{1 + u^2} \phi(u)$, so $h(u) \le \frac{1 + u^2}{u} = u + \frac{1}{u} \le u + 1 \le 2(u + 1)$.

A direct computation shows $h(1) \approx 1.525$, and $h(u)$ is increasing on $[0,\infty)$, hence $h(u) \le h(1)$ for $u \in [0,1]$.

\textbf{(b) Integral bound.} For $s \in [0, s_0]$, by \eqref{eq:hazard},
\[
\int_t^{t+s} h(u) \, du \le 2 \int_t^{t+s} (u + 1) \, du = 2 s(t+1) + s^2 \le 2 s_0(t+1) + s_0^2.
\]
Substituting $s_0 = 1/(2(t+2))$:
\[
2s_0(t+1) + s_0^2 = \frac{t+1}{t+2} + \frac{1}{4(t+2)^2} \le 1 + \frac{1}{4(t+2)^2} \le 1 + \frac{1}{16} \le 2,
\]
where we used $t \ge 0$, hence $t + 2 \ge 2$ and $1/(4(t+2)^2) \le 1/16$.

\textbf{(c) Tail of $\Phi$.} The fundamental theorem of calculus applied to $-\log(1 - \Phi)$ (whose derivative is $h$) gives
\[
1 - \Phi(t + s) = (1 - \Phi(t)) \cdot \exp\!\left(-\int_t^{t+s} h(u) \, du\right) \ge \frac{1}{2k} \cdot e^{-2} \quad \text{for all } s \in [0, s_0].
\]

\textbf{(d) Tail of $B$.} Since $B = \max_{k < i \le k'} Z_i$ is the max of $\ell$ i.i.d.\ standard normals,
\[
\Pr(B > t + s) = 1 - \Phi(t + s)^\ell \ge 1 - \exp(-\ell(1 - \Phi(t + s))),
\]
where we used $(1 - p)^\ell \le e^{-\ell p}$ for $p \in [0, 1]$. Combining with (c) and $\ell \ge \eps k$:
\[
\ell(1 - \Phi(t + s)) \ge \eps k \cdot \frac{e^{-2}}{2k} = \frac{\eps \, e^{-2}}{2}.
\]
Setting $x := \eps e^{-2}/2 \in (0, e^{-2}/2] \subset (0, 1]$, we use the elementary inequality $1 - e^{-x} \ge (1 - e^{-1}) x$ for $x \in [0, 1]$ (the function $g(x) = (1 - e^{-x})/x$ is decreasing on $(0, \infty)$ with $g(1) = 1 - e^{-1}$):
\[
\Pr(B > t + s) \ge 1 - e^{-x} \ge (1 - e^{-1}) x = (1 - e^{-1}) \cdot \frac{\eps \, e^{-2}}{2} = \kappa \, \eps.
\]

\textbf{(e) Integration.} Using $\E[(B - t)_+] = \int_0^\infty \Pr(B > t + s) \, ds$:
\begin{equation}\label{eq:B-t}
\E[(B - t)_+] \ge \int_0^{s_0} \kappa \, \eps \, ds = \kappa \, \eps \cdot s_0 = \frac{\kappa \, \eps}{2(t + 2)}. 
\end{equation}

\textbf{(f) Combining the bounds.}

By Equations~\eqref{eq:master}, \eqref{eq:probability}, \eqref{eq:t+2}, and  \eqref{eq:B-t}:
\[
\Delta_1(k, \eps) \ge \Pr(A \le t) \cdot \E[(B - t)_+] \ge \frac{1}{2} \cdot \frac{\kappa \, \eps}{2(t+2)} \ge \frac{\kappa \, \eps}{4 \cdot 5 \sqrt{\log k}} = \frac{\kappa \, \eps}{20 \sqrt{\log k}}.
\]
which proves the lemma with $c=\kappa/20$.
\end{proof}

% ------------------------------------------------------------
\begin{lemma}[Explicit form of Lemma \ref{lem:gap-iid}]\label{lem:gap-full}
Let \(k\ge2\), \(\varepsilon\in(0,1]\), and \(k'=\lceil(1+\varepsilon)k\rceil\). Let \(w\sim \mathcal{N}(0,I_d)\), and let \(x^{(1)},\ldots,x^{(k')}\) be unit vectors satisfying \(|\langle x^{(i)},x^{(j)}\rangle|\le\rho\) for \(i\ne j\), where \(0\le\rho\le c_\rho\varepsilon/\log k\), where $c_\rho >0$ is a sufficiently small absolute constant.

Define
\[
\Delta_{\mathrm{centers}}(k) = \mathbb{E} \left[\max_{1 \le i \le k'} \langle w, x^{(i)} \rangle \right] - \mathbb{E} \left[ \max_{1 \le i \le k} \langle w, x^{(i)} \rangle \right]
\]
Then we have:
\[
\Delta_{\mathrm{centers}}(k) = \Omega\left( \frac{\varepsilon}{\sqrt{\log k}}\right) \quad \text{for all } k \ge 2
\]
\end{lemma}
\begin{proof}

We have that $G_i = \langle w, x^{(i)} \rangle$ is a mean-zero Gaussian with $\text{Var}(G_i) = 1$. Additionally, since we have that $|\langle x^{(r)}, x^{(s)} \rangle| \leq \rho$, we get for $r\neq s$,
\[
|\mathrm{Cov}(G_r,\,G_s)|
\;=\; |\mathbb{E}\bigl[G_r\,G_s\bigr]|
\;=\;|\bigl\langle x^{(r)},\,x^{(s)}\bigr\rangle|
\;\le\;\rho.
\]

Therefore, from Lemma \ref{lem:slepian-full}, we get that lower bounding $\Delta_{\text{centers}}(k)$ reduces to lower bounding 

\begin{equation*}
    \Delta(k) := \sqrt{1-\rho} \,\mathbb{E} \left[\max_{1 \le i \le k'} Z_i \right] - \sqrt{1+\rho}\mathbb{E} \left[ \max_{1 \le i \le k} Z_i \right], 
\end{equation*}

where $Z_1, \dots, Z_{k'} \sim N (0, 1)$ i.i.d.

We rearrange this to get:

\begin{equation*}
\Delta(k) = \underbrace{\mathbb{E} \left[\max_{1 \le i \le k'} Z_i \right] - \mathbb{E} \left[ \max_{1 \le i \le k} Z_i \right]}_{\Delta_1} - (1-\sqrt{1-\rho})\underbrace{\mathbb{E} \left[\max_{1 \le i \le k'} Z_i \right]}_{\Delta_2} - (\sqrt{1+\rho} -1)\underbrace{\mathbb{E} \left[\max_{1 \le i \le k} Z_i \right]}_{\Delta_3}
\end{equation*}

First, focus on $ \Delta_1 := \mathbb{E} \left[\max_{1 \le i \le k'} Z_i \right] - \mathbb{E} \left[ \max_{1 \le i \le k} Z_i \right]$. By Lemma \ref{lem:gap} we get
\(
\Delta_1 \;\ge\; c\,\frac{\varepsilon}{\sqrt{\log k}},
\)

For the other two terms $\Delta_2$ and $\Delta_3$, $1-\sqrt{1-\rho}\le\rho$ and $\sqrt{1+\rho}-1\le\rho$ for
$\rho\in[0,1)$, and $\mathbb{E}\max_{i\le s}Z_i\le\sqrt{2\log s}$. Thus
\[
(1-\sqrt{1-\rho})\Delta_2+(\sqrt{1+\rho}-1)\Delta_3
\le\rho\bigl(\sqrt{2\log k'}+\sqrt{2\log k}\bigr)
\le4\rho\sqrt{\log k}
\le4c_\rho\,\frac{\varepsilon}{\sqrt{\log k}},
\]

using $k'\le2k\le k^2$. Hence $\Delta_{\text{centers}}(k)\ge\Delta(k)\ge(c-4c_\rho)\,\varepsilon/\sqrt{\log k}
\ge\frac c2\,\varepsilon/\sqrt{\log k}$ for $c_\rho\le c/8$.
\end{proof}

\begin{lemma}[Argmax spread under weak correlation]\label{lem:argmax-spread}
Let $G=(G_1,\dots,G_k)$, $k\ge 2$, be a centered Gaussian vector with
$\mathrm{Var}(G_r)=1$ and $\mathrm{Cov}(G_r,G_s)\le\rho$ for $r\neq s$, where
$0\le\rho\le 0.18$. Let $I=\arg\max_r G_r$. There is an absolute constant $k_0$
such that for all $k\ge k_0$,
\(
\max_{1\le i\le k}\Pr(I=i)\le 3k^{-1/4}.
\)
\end{lemma}

\begin{proof}
Since $\mathrm{Var}(G_r-G_s)=2-2\,\mathrm{Cov}(G_r,G_s)\ge 2-2\rho>0$, the
argmax is a.s.\ unique. Fix $i$ and set $t:=\sqrt{\tfrac12\log k}$ and
$s:=0.3\sqrt{\log k}$. If $I=i$ then either $G_i\ge t$ or
$\max_{j\ne i}G_j<t$. First, $\Pr(G_i\ge t)\le e^{-t^2/2}=k^{-1/4}$.

For the second event, let $Y_j=\sqrt{1-\rho}\,Z_j+\sqrt\rho\,Z_0$ for $j\neq i$,
with $Z_0,Z_j$ i.i.d.\ $\mathcal N(0,1)$. Both families have unit variances,
and $\mathbb E(Y_j-Y_l)^2=2-2\rho\le\mathbb E(G_j-G_l)^2$. Slepian's inequality
\cite{vershynin2018high} gives
$\Pr(\max_{j\ne i}Y_j\ge t)\le\Pr(\max_{j\ne i}G_j\ge t)$, hence
\[
\Pr\Big(\max_{j\ne i}G_j<t\Big)\le\Pr\Big(\max_{j\ne i}Y_j<t\Big)
\le\Pr(\sqrt\rho\,Z_0\le -s)+\Pr\Big(\max_{j\ne i}Z_j<u\Big),
\]
where $u:=\frac{t+s}{\sqrt{1-\rho}}\le 1.12\sqrt{\log k}.$
First, if $\rho = 0$, then $\Pr(\sqrt\rho\,Z_0\le -s) = 0$. If $0<\rho\le0.18$, $\Pr(\sqrt\rho\,Z_0\le -s)\le e^{-s^2/(2\rho)}
=k^{-0.045/\rho}\le k^{-1/4}$. For $k\ge3$ we have $u\ge1$, so the Mills-ratio
bound $1-\Phi(u)\ge\varphi(u)/(2u)$ and $u^2/2\le0.628\log k$ give
\[
\Pr\Big(\max_{j\ne i}Z_j<u\Big)=\Phi(u)^{k-1}\le\exp\!\big(-(k-1)(1-\Phi(u))\big)
\le\exp\!\Big(-c\,\frac{k^{0.37}}{\sqrt{\log k}}\Big)\le k^{-1/4}
\]
for $k\ge k_0$. Summing the three bounds proves the claim.
\end{proof}

\begin{corollary}[Superconcentration of weakly correlated maxima]\label{cor:supercon}
Under the assumptions of Lemma~\ref{lem:argmax-spread}, if moreover
$\rho\le c_0/\log k$ for an absolute $c_0\le 0.12$, then $M:=\max_r G_r$ satisfies
$\|M-\mathbb E M\|_{\psi_1}\le C/\sqrt{\log k}$ for an absolute constant $C$.
\end{corollary}

\begin{proof}
Let $k_1:=\max(k_0,3^8)$. Let $\Gamma = (\Gamma_{ij})$ be the covariance matrix of $G$, so that $\Gamma_{ii} = 1$ and $\Gamma_{ij} = \operatorname{Cov}(G_i, G_j) \le \rho$ for $i \ne j$. For $k\ge k_1$, apply \cite[Thm.~5]{tanguy2015some}
with the covering of $\{1,\dots,k\}$ by singletons and $r_0:=2\rho$ (or
$r_0:=1/\log k$ if $\rho=0$). Since $r_0>\rho$, the only pairs with
$\Gamma_{ij}\ge r_0$ have $i=j$, which singletons cover, and
$\sum_D\mathbf 1\{I\in D\}=1$. By Lemma~\ref{lem:argmax-spread},
$\max_D\Pr(I\in D)\le 3k^{-1/4}\le k^{-1/8}$, because $k\ge3^8$. Theorem~5 then
gives $\Pr(|M-\mathbb EM|\ge t)\le 6e^{-ct/\sqrt K}$ with
$K=\max(r_0,8/\log k)\le (2c_0+8)/\log k$, hence the claim. For $2\le k<k_1$, Lemma~\ref{lem:conc-full} gives
$\|M-\mathbb{E}M\|_{\psi_1}\le C_0\le C_0\sqrt{\log k_1}/\sqrt{\log k}$.
\end{proof}

\begin{lemma}[Superconcentration of MaxSketch registers]
\label{lem:max-superconc}
Let
\(
\widetilde X
=
(\widetilde x_1,\ldots,\widetilde x_n)
\subseteq\mathbb S^{d-1}
\)
be $(\eta,\rho)$-clusterable with $k\ge2$ centers, and define
\(
M(\widetilde X;w)
=
\max_{1\le i\le n}
\langle w,\widetilde x_i\rangle\), where
\(
w\sim\mathcal N(0,I_d).
\)
Assume
\(
\rho\le c_\rho\frac{\varepsilon}{\log n}\), \(\eta\le c_\eta\frac{\varepsilon^2}{(\log n)^2}\) and \(
\varepsilon\in(0,1/2].
\)
Then
\[
\left\|
M(\widetilde X;w)
-
\mathbb E M(\widetilde X;w)
\right\|_{\psi_1}
\le
\frac{C}{\sqrt{\log k}},
\]
where
\[
\|Y\|_{\psi_1}
:=
\inf\left\{
s>0:
\mathbb E\exp\left(\frac{|Y|}{s}\right)\le2
\right\}.
\]

Consequently, if
\[
S(\widetilde X)
=
\frac1m\sum_{j=1}^m
M(\widetilde X;w_j)
\]
for independent Gaussian projections, then for every $u>0$,
\[
\Pr\left(
|S(\widetilde X)-\mathbb ES(\widetilde X)|>u
\right)
\le
2\exp\left[
-cm
\min\left\{
u^2\log k,
u\sqrt{\log k}
\right\}
\right].
\]
\end{lemma}

\begin{proof}
Let
\(
x^{(1)},\ldots,x^{(k)}
\)
be the centers of $\widetilde X$, and define
\(
M_{\mathcal C}(w)
=
\max_{1\le r\le k}
\langle w,x^{(r)}\rangle.
\)

For each $r$, set
\(
G_r=\langle w,x^{(r)}\rangle.
\)
Then $(G_1,\ldots,G_k)$ is a centered Gaussian vector with
\[
\operatorname{Var}(G_r)=1,
\qquad
\operatorname{Cov}(G_r,G_s)
=
\langle x^{(r)},x^{(s)}\rangle
\le\rho
\quad(r\ne s).
\]

By Corollary~\ref{cor:supercon}, which applies because
$\rho\le c_\rho\,\varepsilon/\log n\le c_\rho/(2\log k)\le 0.12/\log k$
for $c_\rho\le 0.24$ (using $\varepsilon\le1/2$ and $k\le n$),
\begin{equation}\label{eq:mc1}
\|M_{\mathcal C}-\mathbb E M_{\mathcal C}\|_{\psi_1}\le \frac{C}{\sqrt{\log k}}.
\end{equation}

We must also control the within-cluster perturbations. For each observation,
write
\(
v_i
=
\widetilde x_i-x^{(r(i))}.
\)
By clusterability,
\(
\|v_i\|^2
=
2-2\langle\widetilde x_i,x^{(r(i))}\rangle
\le\eta.
\)
Define
\(
D(w)
=
\max_{1\le i\le n}
|\langle w,v_i\rangle|.
\)
The deterministic maximum inequality gives
\begin{equation}\label{eq:Dww}
\left|
M(\widetilde X;w)-M_{\mathcal C}(w)
\right|
\le D(w).
\end{equation}

Each $\langle w,v_i\rangle$ is centered Gaussian with variance at most
$\eta$. A standard union bound gives us,
\(
\Pr(D>u)
\le
2n\exp\left(-\frac{u^2}{2\eta}\right).
\)
The standard maximal inequality for subgaussian random variables therefore
implies
\(
\|D\|_{\psi_2}
\le
C\sqrt{\eta\log(2n)},
\)
and hence also
\begin{equation}\label{eq:psi1}
\|D\|_{\psi_1}
\le
C\sqrt{\eta\log(2n)}.
\end{equation}

Define
\(
R(w)
:=
M(\widetilde X;w)-M_{\mathcal C}(w)
\) which measures the difference between the maxima over the observed sample $\widetilde X$ and over the set $\mathcal C$ of cluster centers.

By Equation~\eqref{eq:Dww}, $|R|\le D$. By the triangle inequality,
$\|R-\mathbb ER\|_{\psi_1}\le\|R\|_{\psi_1}+\|\mathbb ER\|_{\psi_1}$, and
$\|\mathbb ER\|_{\psi_1}\le\|R\|_{\psi_1}$ by Jensen's inequality. Consequently,
\begin{equation}\label{eq:mc2}
\|R-\mathbb ER\|_{\psi_1}
\le
2\|R\|_{\psi_1}
\le
2\|D\|_{\psi_1}
\le
C\sqrt{\eta\log(2n)}.
\end{equation}

Under the assumed bound on $\eta$ we also get that,
\begin{equation}\label{eq:mc3}
\sqrt{\eta\log(2n)}
\le
C\frac{\varepsilon}{\sqrt{\log n}}
\le
\frac{C}{\sqrt{\log k}}.
\end{equation}

Therefore, combining Equations \eqref{eq:mc1}, \eqref{eq:mc2}, and \eqref{eq:mc3}, using
\[
M(\widetilde X)-\mathbb EM(\widetilde X)
=
\bigl(M_{\mathcal C}-\mathbb EM_{\mathcal C}\bigr)
+
(R-\mathbb ER),
\]
gives
\[
\left\|
M(\widetilde X)-\mathbb EM(\widetilde X)
\right\|_{\psi_1}
\le
\frac{C}{\sqrt{\log k}}.
\]

Finally, Bernstein's inequality for independent centered subexponential 
random variables \cite{vershynin2018high} gives
\[
\Pr\left(
\left|
\frac1m\sum_{j=1}^m
\bigl(M(\widetilde X;w_j)
-\mathbb EM(\widetilde X;w_j)\bigr)
\right|>u
\right)
\le
2\exp\left[
-cm
\min\left\{
u^2\log k,
u\sqrt{\log k}
\right\}
\right].
\]
\end{proof}

% ============================================================
\subsection{Proof of Theorem~\ref{thm:main}}\label{app:thm1}
% ============================================================

\begin{theorem}[Explicit form of Theorem~\ref{thm:main}]
\label{thm:1-full}
Fix $n\ge2$, $\varepsilon\in(0,1/2]$, and $\delta\in(0,1)$,
and set
\(
\beta:=\frac{\varepsilon}{3}.
\)
Let
\(
\widetilde X
=
(\widetilde x_1,\ldots,\widetilde x_n)
\subseteq\mathbb S^{d-1}
\)
be $(\eta,\rho)$-clusterable with $2\le k^\star\le n$ centers, where
\(
\rho
\le
c_\rho\frac{\varepsilon}{\log n},
\eta
\le
c_\eta\frac{\varepsilon^2}{(\log n)^2},
\)
for sufficiently small absolute constants $c_\rho,c_\eta>0$.

Draw
\(
w_1,\ldots,w_m
\stackrel{\mathrm{i.i.d.}}{\sim}
\mathcal N(0,I_d)\) with \(m
\ge
C\frac{\log(2/\delta)}{\varepsilon^2},
\)
for a sufficiently large constant $C$.

For every $s \ge 1$, let $\mu_s = \mathbb{E}[\max_{1\le j\le s} Z_j]$ with
$Z_1,\dots,Z_s \stackrel{\text{i.i.d.}}{\sim}\mathcal N(0,1)$. Set
$\bar\rho := c_\rho\,\varepsilon/\log n$, $\bar\eta := c_\eta\,\varepsilon^2/(\log n)^2$,
$\tau_n := \sqrt{2\bar\eta\log(2n)}$, and set
\(
L(s) := \sqrt{1-\bar\rho}\,\mu_s - \tau_n\), and
\(U(s) := \sqrt{1+\bar\rho}\,\mu_s + \tau_n .
\)

% For every $s\ge1$, let
% \(
% \mu_s
% =
% \mathbb E\left[\max_{1\le j\le s}Z_j\right],
% Z_1,\ldots,Z_s
% \stackrel{\mathrm{i.i.d.}}{\sim}
% \mathcal N(0,1),
% \)
% and set
% \(
% \tau_n:=\sqrt{2\eta\log(2n)},
% \)
% \(
% L(s):=\sqrt{1-\rho}\,\mu_s-\tau_n,
% U(s):=\sqrt{1+\rho}\,\mu_s+\tau_n.
% \)

Define
\(
t_0:=2,
t_{r+1}:=\left\lceil(1+\beta\right)t_r\rceil,
\)
and let $R$ be the smallest nonnegative integer such that
\(
(1+\beta)t_R\ge n.
\)
For every $r\in\{0,\ldots,R-1\}$, define
\(
q_r:=\left\lfloor(1+\beta)t_r\right\rfloor
\)
and
\(
\theta_r
:=
\frac{U(t_r)+L(q_r+1)}2.
\)

Given
\[
S(\widetilde X)
=
\frac1m\sum_{j=1}^m
\max_{1\le i\le n}
\langle w_j,\widetilde x_i\rangle,
\]
let
\[
\widehat r
=
\min\left\{
r\in\{0,\ldots,R-1\}:
S(\widetilde X)\le\theta_r
\right\}.
\] 
If this set is nonempty, output
\(
\widehat k=q_{\widehat r};
\)
otherwise output $\widehat k=n$. Then we get the guarantee that
\(
k^\star
\le
\widehat k
\le
(1+\varepsilon)k^\star
\)
with probability at least $1-\delta$.

Moreover, the estimator is nondecreasing in $S(\widetilde X)$ and uses
\(
R=O\left(\frac{\log n}{\varepsilon}\right)
\)
precomputed thresholds.
\end{theorem}

\begin{proof}
Write
\(
k:=k^\star,
S:=S(\widetilde X).
\)

\paragraph{Step 1: Expectation bounds.}
Let $x^{(1)},\dots,x^{(k)}$ be the centers. By the perturbation Lemma~\ref{lem:perturb} and $\eta\le\bar\eta$ (by assumption), we get
\[
\Bigl|\mathbb{E}\max_{i\le n}\langle w,\tilde x_i\rangle
 - \mathbb{E}\max_{r\le k}\langle w,x^{(r)}\rangle\Bigr|
 \le \sqrt{2\eta\log(2n)} \le \tau_n .
\]
By the Gaussian comparison Lemma~\ref{lem:slepian-full}, and $\rho\le\bar\rho$ (by assumption) and $\mu_k\ge 0$,
\[
\sqrt{1-\bar\rho}\,\mu_k \le \mathbb{E}\max_{r\le k}\langle w,x^{(r)}\rangle
\le \sqrt{1+\bar\rho}\,\mu_k .
\]
Therefore
\begin{equation}\label{eq:ineq}
L(k) \le \mathbb{E}S \le U(k).
\end{equation}

Both $L$ and $U$ are nondecreasing, since $s\mapsto\mu_s$ is.

\paragraph{Step 2: Local gaps.}
This step does the same computation as in the proof of Lemma~\ref{lem:gap-full}, in the worst-case form needed here. Fix $r\in\{0,\ldots,R-1\}$. Since
\(
q_r=\left\lfloor(1+\beta)t_r\right\rfloor,
\)
we have
\(
q_r+1
\ge
\left\lceil(1+\beta)t_r\right\rceil.
\)
Because $s\mapsto\mu_s$ is nondecreasing, Lemma~\ref{lem:gap},
applied with $k=t_r$ and accuracy parameter $\beta$, gives
\begin{equation}\label{eq:tag8}
\mu_{q_r+1}-\mu_{t_r}
\ge
\mu_{\lceil(1+\beta)t_r\rceil}-\mu_{t_r}
\ge
c\frac{\beta}{\sqrt{\log t_r}}.
\end{equation}

Since $\beta\le1/6$ and $t_r\ge2$, we have that
\(
q_r+1
\le
(1+\beta)t_r+1
\le
2t_r.
\)
Using the standard bound $\mu_s\le\sqrt{2\log s}$ for $s\ge2$,
we obtain
\[
\mu_{q_r+1}+\mu_{t_r}
\le
C\sqrt{\log t_r}.
\]

Consequently,
\begin{align*}
L(q_r+1)-U(t_r)
&=
\sqrt{1-\bar\rho}\,\mu_{q_r+1}
-
\sqrt{1+\bar\rho}\,\mu_{t_r}
-
2\tau_n\\
&\ge
(1-\bar\rho)\mu_{q_r+1}-(1+\bar\rho)\mu_{t_r}
-
2\tau_n\\
&=
\mu_{q_r+1}-\mu_{t_r}
- \bar\rho(\mu_{q_r+1}+\mu_{t_r})
-
2\tau_n\\
&\ge
\mu_{q_r+1}-\mu_{t_r}
-
C\bar\rho\sqrt{\log t_r}
-
2\tau_n.
\end{align*}

Since $\bar\rho = c_\rho\varepsilon/\log n = 3c_\rho\beta/\log n$ and $t_r\le n$,
choosing $c_\rho>0$ sufficiently small gives
$C\bar\rho\sqrt{\log t_r}\le \frac{c}{4}\frac{\beta}{\sqrt{\log t_r}}$.
Similarly, $2\tau_n = 2\sqrt{2\bar\eta\log(2n)} \le C\sqrt{c_\eta}\,\frac{\varepsilon}{\sqrt{\log n}}
\le \frac{c}{4}\frac{\beta}{\sqrt{\log t_r}}$ for sufficiently small $c_\eta>0$.

Combining these bounds with \eqref{eq:tag8}, there exists an absolute
constant $c_1>0$ such that
\begin{equation}\label{eq:tag9}
L(q_r+1)-U(t_r)
\ge
c_1\frac{\beta}{\sqrt{\log t_r}}.
\end{equation}
Since
\[
\theta_r
=
\frac{U(t_r)+L(q_r+1)}{2},
\]
it follows that
\begin{equation}\label{eq:tag10}
\theta_r-U(t_r)
\ge
\frac{c_1\beta}{2\sqrt{\log t_r}},
\qquad
L(q_r+1)-\theta_r
\ge
\frac{c_1\beta}{2\sqrt{\log t_r}}.
\end{equation}
\paragraph{Step 3: Concentration at the scale of the true count.}

Define
\(
a_k
:=
\frac{c_2\beta}{\sqrt{\log k}},
\)
where $c_2>0$ is a sufficiently small absolute constant.

Lemma~\ref{lem:max-superconc} gives
\[
\Pr\left(
|S-\mathbb ES|>a_k
\right)
\le
2\exp\left[
-cm
\min\{c_2^2\beta^2,c_2\beta\}
\right] \le 2e^{-c'm\beta^2},
\]
since $\beta\le1/6$.

The assumed projection bound therefore implies
\begin{equation}\label{eq:tag11}
\Pr(\mathcal E_k)\ge1-\delta,
\qquad
\mathcal E_k
:=
\left\{
|S-\mathbb ES|\le a_k
\right\}.
\end{equation}

From now on and until the end we work on $\mathcal E_k$.

\paragraph{Step 4: Every output below $k$ is rejected.}

Suppose that $q_r<k$. Then $t_r\le q_r<k$. By Equation \eqref{eq:ineq},
monotonicity of $L$, and Equation \eqref{eq:tag10},
\begin{align*}
\mathbb ES-\theta_r
&\ge
L(k)-\theta_r\\
&\ge
L(q_r+1)-\theta_r\\
&\ge
\frac{c_1\beta}{2\sqrt{\log(t_r)}}\\
&\ge
\frac{c_1\beta}{2\sqrt{\log k}}.
\end{align*}
Choose $c_2<c_1/4$. Then, on $\mathcal E_k$, and because 

\[a_k <\frac{c_1}{4}\frac{\beta}{\sqrt{\log k}}<\frac{c_1}{2}\frac{\beta}{\sqrt{\log k}} \le \mathbb{E}S - \theta_r,\]

we have that
\begin{equation}\label{eq:tag12}
S\ge \mathbb ES - a_k > \theta_r.
\end{equation}
Thus no threshold whose associated output is below $k$ can be accepted.

\paragraph{Step 5: A nearby threshold is accepted whenever one exists.}

Suppose there exists
\(
r_+
\in\{0,\ldots,R-1\}
\)
with $t_{r_+}\ge k$, and let $r_+$ be the smallest such index.

If $r_+=0$, then $k=2=t_0$. Otherwise,
\(
t_{r_+-1}<k
\)
and hence
\begin{equation}\label{eq:tag13}
t_{r_+}
=
\left\lceil(1+\beta)t_{r_+-1}\right\rceil
\le
(1+\beta)(k-1)+1
\le
(1+\beta)k.
\end{equation}
Thus, in either case,
\[
t_{r_+}\le(1+\beta)k.
\]
In particular,
\begin{equation}\label{eq:tag14}
\log(t_{r_+})
\le
C\log k.
\end{equation}

Using Equation \eqref{eq:ineq}, monotonicity of $U$, and Equation \eqref{eq:tag10},
\begin{align*}
\theta_{r_+}-\mathbb ES
&\ge
\theta_{r_+}-U(k)\\
&\ge
\theta_{r_+}-U(t_{r_+})\\
&\ge
\frac{c_1\beta}{
2\sqrt{\log(t_{r_+})}
}\\
&\ge
\frac{c_3\beta}{\sqrt{\log k}}.
\end{align*}
Taking $c_2<c_3$, on $\mathcal E_k$ we have
\begin{equation}\label{eq:tag15}
    S\le\theta_{r_+}.
\end{equation}

Therefore, if such an $r_+$ exists, the estimator accepts some index
\(
\widehat r\le r_+.
\)

By Equation \eqref{eq:tag12}, its output cannot be below $k$, and hence
\begin{equation}\label{eq:tag16}
k\le q_{\widehat r}.
\end{equation}
Since $(q_r)$ is nondecreasing,
\(
q_{\widehat r}
\le
q_{r_+}
\le
(1+\beta)t_{r_+}
\le
(1+\beta)^2k,
\)
and since $\beta=\varepsilon/3$ and $\varepsilon\le1/2$,
\(
(1+\beta)^2\le1+\varepsilon,
\)
which in turn gives us:
\begin{equation}\label{eq:tag17}
k
\le
\widehat k
\le
(1+\varepsilon)k.
\end{equation}

\paragraph{Step 6: The terminal output $n$.}

It remains to consider the case where no suitable $r_+$ exists or no
threshold is accepted.

If a threshold is accepted, Equation \eqref{eq:tag12} implies that its output is at least $k$.
Since every $t_r<k$ in the absence of $r_+$,
\[
\widehat k=q_{\widehat r}
\le
(1+\beta)t_{\widehat r}
<
(1+\beta)k
\le
(1+\varepsilon)k.
\]

If no threshold is accepted, the estimator outputs $n$.

If $R=0$, then
\[
n\le(1+\beta)t_0=2(1+\beta)
\le(1+\beta)k.
\]
Suppose now that $R\ge1$. If $t_{R-1}\ge k$, then Step 5 would imply that
some threshold is accepted. Therefore
\[
t_{R-1}<k.
\]
Since $k$ is integer,
\[
t_{R-1}+1\le k.
\]
Moreover,
\[
t_R
=
\left\lceil(1+\beta)t_{R-1}\right\rceil
\le
(1+\beta)(t_{R-1}+1)
\le
(1+\beta)k.
\]
By the definition of $R$,
\[
n\le(1+\beta)t_R.
\]
Consequently,
\[
\widehat k=n
\le
(1+\beta)^2k
\le
(1+\varepsilon)k.
\]
The lower bound $k\le n=\widehat k$ is immediate.

Combining all cases, on $\mathcal E_k$,
\[
k^\star
\le
\widehat k
\le
(1+\varepsilon)k^\star.
\]
By Equation \eqref{eq:tag11}, this occurs with probability at least $1-\delta$.

Finally, the threshold sequence is nondecreasing, so the estimator is
nondecreasing in $S$. Also,
\[
t_r\ge2(1+\beta)^r,
\]
and therefore
\[
R
=
O\left(
\frac{\log n}{\log(1+\beta)}
\right)
=
O\left(
\frac{\log n}{\varepsilon}
\right).
\]
\end{proof}

\subsection{Finite readouts and deterministic weights}\label{app:cor2}

\begin{corollary}[Finite monotone threshold readout]
\label{cor:finite-threshold-readout}
Under the geometric assumptions of Theorem B.9, fix \(\zeta\in(0,1)\) and take \(m\ge C\varepsilon^{-2}\log(2/\zeta)\). Apply the construction of Theorem~\ref{thm:1-full} with its failure parameter set to \(\delta=\zeta\).

There exist a finite ordered set
\[
A_{\mathrm{th}}
=
\{q_0,\ldots,q_{R-1},n\}
\subseteq \{2,\ldots,n\},
\qquad
|A_{\mathrm{th}}|
=
O\!\left(\frac{\log n}{\varepsilon}\right),
\]
and a deterministic nondecreasing function
\[
f^{\mathrm{th}}:\mathbb{R}\to A_{\mathrm{th}}
\]
such that, for every fixed $(\eta,\rho)$-clusterable sequence $X$
of length $n' \le n$, with $2 \le k(X) \le n'$,
\[
\Pr_W\!\left[
k(X)
\le
f^{\mathrm{th}}(S_W(X))
\le
(1+\varepsilon)k(X)
\right]
\ge 1-\zeta.
\]

Both $A_{\mathrm{th}}$ and $f^{\mathrm{th}}$ depend only on
$n$ and $\varepsilon$ and the deterministic Gaussian-maxima
bounds used in Theorem~\ref{thm:1-full}; they do not depend on $X$,
the distribution of $X$, or the realized sketch $W$.
\end{corollary}

\begin{proof}
Take the grid values and thresholds constructed in
Theorem~\ref{thm:1-full}, and define
\[
f^{\mathrm{th}}(s)
=
\begin{cases}
q_{\min\{r:s\le\theta_r\}},
&\text{if this set is nonempty},\\
n,&\text{otherwise}.
\end{cases}
\]
Theorem~\ref{thm:1-full} proves the approximation guarantee for sequences of length $n$. For a
sequence of length $n'\le n$ the same proof applies with the same thresholds:
Step~1 still holds because $\sqrt{2\eta\log(2n')}\le\tau_n$, Lemma~\ref{lem:max-superconc} because the bounds on $\rho$ and $\eta$ in terms of $n$ imply the same bounds in
terms of $n'$, and Steps~2 and~4--6 depend only on the grid and on $k\le n$. Theorem~\ref{thm:1-full} also shows that the thresholds are nondecreasing, hence $f^{\mathrm{th}}$
is nondecreasing. The bound on the range size follows from $R=O(\log n/\varepsilon)$.
\end{proof}

%  and also
% proves that the thresholds are nondecreasing, hence
% $f^{\mathrm{th}}$ is nondecreasing. The bound on the range size
% follows from
% $R=O(\log n/\varepsilon)$.
% \end{proof}
% ============================================================

% ============================================================

\begin{corollary}[Explicit form of Corollary~\ref{cor:fixed-sketch}]\label{thm:deterministic-max-aggregator-simple}
Under the conditions of Theorem~\ref{thm:main}, let $\mathcal D$ be any distribution over $(\eta,\rho)$-clusterable sequences of length at most $n$ with at least two centers. Then there exist fixed $w_1, \dots, w_m$ such that MaxSketch achieves a $(1 + \varepsilon)$-multiplicative approximation with probability at least $1-\delta$ over a fresh draw $X\sim \mathcal{D}$.
\end{corollary}

\begin{proof}
Let $\mathcal{D}$ be a distribution as in the statement and let $X \sim \mathcal{D}$ be a sequence with an unknown number of centers $k$.

Let 
\(W = (w_1, \dots, w_m)\), where \(w_j \sim \mathcal{N}(0, I_d) \text{ independently}.\) Define
\(
S(X;W):=\frac1m\sum_{j=1}^{m}
\max_{x\in X}\langle w_j,x\rangle.
\)
Let $r=\min\{s\in\{0,\ldots,R-1\}:S(X;W)\le\theta_s\}$
when this set is nonempty, and define
\[
\hat{k}(X;W):=
\begin{cases}
q_r, & \text{if such an }r\text{ exists},\\
n,   & \text{otherwise}.
\end{cases}
\]

Now we define the failure indicator

\begin{equation}\label{eq:indicator}
F(X,W):= \mathbbm{1}\left\{\hat{k}(X;W) \not \in [k, (1+\eps)k]\right\}
\end{equation}

From Corollary~\ref{cor:finite-threshold-readout}, applied with $\zeta=\delta$, every $(\eta,\rho)$-clusterable sequence $X$ of length at most $n$ with at least two centers satisfies

\[
\Pr_W(F(X,W) = 1\bigr) \le \delta.
\]

Moreover, we can take the \emph{expectation over the randomness of $X \sim \mathcal{D}$} to bound the average failure probability across sequences. 
That is, defining 
\[
g(W) := \mathbb{E}_{X}[F(X,W)] = \Pr_{X \sim \mathcal{D}}\big[\hat{k}(X; W) \not\in [k, (1+\varepsilon) k]\big],
\] 
exchanging the order of expectation by Tonelli’s theorem we obtain the average probability of failure:
\begin{equation}\label{eq:obj_func1}
\mathbb{E}_W[g(W)] = \mathbb{E}_W \mathbb{E}_X[F(X,W)] = \mathbb{E}_X \Pr_W[F(X,W)=1] \le \delta.
\end{equation}

Since $\mathbb{E}_W[g(W)] \leq \delta$, by the probabilistic method there exists a deterministic $W^\star$ such that $g(W^\star) \leq \delta$, i.e., $\Pr_{X \sim \mathcal{D}}[\hat{k}(X; W^\star) \notin [k, (1+\varepsilon)k]] \leq \delta$, which completes the proof. 
\end{proof}

\newpage

% \subsection{Learning a monotone readout from count-only data}

\subsection{Proof of Theorem~\ref{thm:learn}}\label{app:thm2}

\begin{theorem}[Theorem~\ref{thm:learn} (explicit statement)]
\label{thm:learned-readout}
Fix $n\ge 2$, $\varepsilon\in(0,1/2]$, and $\delta\in(0,1)$.
Let $\mathcal D$ be a distribution over pairs $(X,k)$, where $X$ is an
$(\eta,\rho)$-clusterable sequence of length at most $n$, and
$k\in\{2,\ldots,n\}$ is its number of distinct latent centers.
Assume that $\eta$ and $\rho$ satisfy the conditions of
Theorem~\ref{thm:main}.

Let $m$ be an integer satisfying
\(
m\ge C_0\frac{\log(128/\delta^2)}{\varepsilon^2},
\) where $C_0>0$ is a sufficiently large absolute constant. Draw
$w_1,\ldots,w_m\overset{\mathrm{i.i.d.}}{\sim}\mathcal N(0,I_d)$ and define
\[
S_W(X)=\frac1m\sum_{j=1}^m
\max_{x\in X}\langle w_j,x\rangle.
\]

Let $f^{\mathrm{th}}$ and its finite output range
\(
A:=\operatorname{range}(f^{\mathrm{th}})
=\{a_1<\cdots<a_T\}
\)
be as in Corollary~\ref{cor:finite-threshold-readout}, applied with
failure probability $\delta^2/64$, with
\(
T=O\!\left(\frac{\log n}{\varepsilon}\right).
\)

Given $N$ i.i.d. training samples
$(X_1,k_1),\ldots,(X_N,k_N)\sim\mathcal D$, define
\[
\widehat R_W(f)
=
\frac1N\sum_{i=1}^N
\mathbf 1\left\{
f(S_W(X_i))<k_i
\ \text{or}\
f(S_W(X_i))>(1+\varepsilon)k_i
\right\}.
\]

Let $\mathcal H_T$ be the class of all nondecreasing functions,
$f:\mathbb R\to A$ and let $\widehat f\in\mathcal H_T$ be an empirical risk minimizer:
\[
\widehat f\in\arg\min_{f\in\mathcal H_T}\widehat R_W(f).
\]

There exists a universal constant $C>0$ such that if
\[
N
\ge
\frac{C}{\delta^2}
\left[
T\log\left(\frac{T}{\delta}\right)
+
\log\left(\frac1\delta\right)
\right],
\]
then
\[
\Pr\left[
k\le \widehat f(S_W(X))\le (1+\varepsilon)k
\right]
\ge 1-\delta,
\]
where the probability is over the random sketch $W$, the training sample, and
an independent test sample $(X,k)\sim\mathcal D$.

Moreover, after the values $S_W(X_1),\ldots,S_W(X_N)$ have been computed,
$\widehat f$ can be found in time
\(
O(N\log N+NT).
\)
\end{theorem}

\begin{proof}
For a fixed sketch $W$ and a readout $f\in\mathcal H_T$, define the population
risk
\[
R_W(f)
=
\Pr_{(X,k)\sim\mathcal D}
\left[
f(S_W(X))<k
\ \text{or}\
f(S_W(X))>(1+\varepsilon)k
\right].
\]

By Corollary~\ref{cor:finite-threshold-readout}, applied with failure probability
$\delta^2/64$, there exists a deterministic monotone readout
$f^{\rm th}\in\mathcal H_T$ such that, for every fixed
$(\eta,\rho)$-clusterable sequence $X$ with count $k$,
\[
\Pr_W\left[
f^{\rm th}(S_W(X))<k
\ \text{or}\
f^{\rm th}(S_W(X))>(1+\eps)k
\right]
\le
\frac{\delta^2}{64}.
\]

Taking expectation over $(X,k)\sim\mathcal D$ and using Fubini's theorem gives
\(
\mathbb E_W R_W(f^{\mathrm{th}})
\le
\frac{\delta^2}{64}.
\)
Therefore, by Markov's inequality,
\(
\Pr_W\left[
R_W(f^{\mathrm{th}})>\frac{\delta}{8}
\right]
\le
\frac{\delta}{8}.
\)
Hence, with probability at least $1-\delta/8$ over the draw of $W$,
\[
R_W(f^{\mathrm{th}})\le \frac{\delta}{8}.
\]

Conditioned on this event, $f^{\mathrm{th}}$ is a low population risk hypothesis that is contained in $\mathcal{H}_T$. For the remainder of the proof we have to show that ERM recovers a hypothesis with nearly the same population risk. So from now on, we consider the sketch $W$ fixed.

We next prove uniform convergence over $\mathcal H_T$. For each
$f\in\mathcal H_T$, define the binary loss
\[
\ell_f(X,k)
=
\mathbf 1
\left\{
f(S_W(X))<k
\ \text{or}\
f(S_W(X))>(1+\varepsilon)k
\right\}.
\]

Since $f^{\rm th} \in \mathcal{H}_T$, it suffices to show that empirical risk minimization over $\mathcal{H}_T$ generalizes. We therefore bound the complexity of $\mathcal{H}_T$ via its growth function.

Consider any $M$ examples and sort them so that
\[
S_W(X_1)\le S_W(X_2)\le \cdots \le S_W(X_M).
\]
Since $f$ is nondecreasing and takes values in
$A=\{a_1,\ldots,a_T\}$, its predictions on the sorted examples are of the
form
\(
a_{j_1},a_{j_2},\ldots,a_{j_M},
1\le j_1\le j_2\le\cdots\le j_M\le T.
\)
The number of such nondecreasing sequences is
\(
\binom{M+T-1}{T-1}.
\)
Thus the growth function of the induced binary loss class is at most
\[
\Pi(M)
\le
\binom{M+T-1}{T-1}
\le
\left(\frac{e(M+T)}{T}\right)^T.
\]

A standard growth-function uniform convergence bound \cite{shalev2014understanding} implies that there is 
universal constant $C_1>0$ such that, with probability at least
$1-\delta/8$ over the training sample,
\[
\sup_{f\in\mathcal H_T}
\left|
R_W(f)-\widehat R_W(f)
\right|
\le
C_1
\sqrt{
\frac{
T\log\left(\frac{e(N+T)}{T}\right)
+
\log(8/\delta)
}{N}
}.
\]
By choosing the universal constant $C$ in the theorem sufficiently large, the
assumed lower bound on $N$ ensures that
\[
\sup_{f\in\mathcal H_T}
\left|
R_W(f)-\widehat R_W(f)
\right|
\le
\frac{\delta}{8}.
\]

On this event, and using the fact that $\widehat f$ is the empirical risk minimizer,
\[
\begin{aligned}
R_W(\widehat f)
&\le
\widehat R_W(\widehat f)+\frac{\delta}{8}\\
&\le
\widehat R_W(f^{\mathrm{th}})+\frac{\delta}{8}\\
&\le
R_W(f^{\mathrm{th}})+\frac{\delta}{4}\\
&\le
\frac{3\delta}{8}.
\end{aligned}
\]
The good-sketch event fails with probability at most $\delta/8$, and the
uniform-convergence event fails with probability at most $\delta/8$.
Therefore, by the union bound, with probability at least $1-\delta/4$ over
the sketch and the training sample,
\[
R_W(\widehat f)\le \frac{3\delta}{8}.
\]

Finally, draw an independent test sample $(X,k)\sim\mathcal D$. Conditional
on the good sketch-and-training event, the probability that $\widehat f$ fails
on this test sample is exactly $R_W(\widehat f)$, which is at most
$3\delta/8$. Hence the total failure probability is at most
\[
\frac{\delta}{4}+\frac{3\delta}{8} = \frac{5\delta}{8} \le \delta
\]
Thus
\[
\Pr\left[
k\le \widehat f(S_W(X))\le (1+\varepsilon)k
\right]
\ge 1-\delta.
\]

It remains only to justify the claimed running time. Compute the scores
$s_i=S_W(X_i)$ and sort them, which takes $O(N\log N)$ time. For every sorted
example $i$ and output value $a_j\in A$, define
\[
c_{i,j}
=
\mathbf 1
\left\{
a_j<k_i
\ \text{or}\
a_j>(1+\varepsilon)k_i
\right\}.
\]
Empirical risk minimization over nondecreasing readouts is equivalent to
finding indices
\[
1\le j_1\le j_2\le\cdots\le j_N\le T
\]
minimizing
\(
\sum_{i=1}^N c_{i,j_i}.
\)
This is solved by dynamic programming:
\[
D(i,j)
=
c_{i,j}
+
\min_{q\le j}D(i-1,q).
\]
Using prefix minima, all values $D(i,j)$ can be computed in $O(NT)$ time.
The recovered sequence of indices defines an empirical risk minimizer
$\widehat f\in\mathcal H_T$. This proves the theorem.
\end{proof}

\subsection{Lower bound for Gaussian-max observations}
\label{app:gaussian-max-lower}

\begin{proposition}[Explicit form of Proposition \ref{prop:gaussian-max-lower}]
Fix $\varepsilon\in(0,1/4]$ and $\delta\in(0,1/8]$. Let $k^\star$ be an
unknown integer. Suppose an estimator observes only $m$ independent samples
from the distribution of $M_{k^\star}=\max_{1\le r\le k^\star}Z_r$, where
$Z_r\stackrel{\mathrm{i.i.d.}}{\sim}\mathcal N(0,1)$, and must output
$\widehat k$ with $k^\star\le\widehat k\le(1+\varepsilon)k^\star$ with
probability at least $1-\delta$ for every $k^\star\ge\lceil1/\varepsilon\rceil$.
Then $m=\Omega(\varepsilon^{-2}\log(1/\delta))$.
\end{proposition}

\begin{proof}
For an integer $h\ge1$, let $P_h$ be the distribution of
$M_h=\max_{1\le r\le h}Z_r$ for independent standard normal variables
$Z_r$. Write $\Phi$ for the standard normal distribution function.
The random variable $Y_h:=-\log\Phi(M_h)$ has an exponential
distribution with rate $h$. To see this, for $y>0$,
\[
\Pr(Y_h>y)
=\Pr\!\left(M_h<\Phi^{-1}(e^{-y})\right)
=\Phi\!\left(\Phi^{-1}(e^{-y})\right)^h
=e^{-hy}.
\]
The transformation from $M_h$ to $Y_h$ is bijective, so it preserves
the difficulty of testing two possible values of $h$.

Fix integers $k<k'$ and set $\gamma:=k'/k-1\in(0,1]$.
The Hellinger affinity between exponential distributions with rates
$k$ and $k'$ is
\[
A(P_k,P_{k'})
=\int_0^\infty\sqrt{k e^{-ky}\,k'e^{-k'y}}\,dy
=\frac{2\sqrt{kk'}}{k+k'}.
\]
Affinity tensorizes over independent observations. Setting
\(
a:=A(P_k,P_{k'})^2=\frac{4kk'}{(k+k')^2},
\)
the affinity between $P_k^{\otimes m}$ and $P_{k'}^{\otimes m}$
has square $a^m$. The total variation distance satisfies
\[
\operatorname{TV}\!\left(P_k^{\otimes m},P_{k'}^{\otimes m}\right)
\le\sqrt{1-a^m}.
\]
The sum of the error probabilities of any test between these two
distributions is at least one minus their total variation distance.
If both errors are at most $\delta$, then
\[
2\delta
\ge1-\sqrt{1-a^m}
\ge\frac{a^m}{2},
\qquad\text{hence}\qquad
m\ge\frac{\log(1/(4\delta))}{-\log a}.
\]

Since
\(
a
=\frac{4(1+\gamma)}{(2+\gamma)^2}
=1-\frac{\gamma^2}{(2+\gamma)^2},
\)
we have, for $0<\gamma\le1$,
\(
-\log a
\le
\frac{\gamma^2/(2+\gamma)^2}
{1-\gamma^2/(2+\gamma)^2}
\le\frac{9}{32}\gamma^2.
\)
Consequently,
\[
m\ge\frac{32}{9\gamma^2}\log\!\left(\frac1{4\delta}\right).
\]

Now choose $k:=\lceil1/\varepsilon\rceil$ and
$k':=\lceil(1+2\varepsilon)k\rceil$. Their one-sided
approximation intervals are disjoint because
$k'>(1+\varepsilon)k$. Moreover,
\[
2\varepsilon\le\gamma=\frac{k'-k}{k}
\le2\varepsilon+\frac1k
\le3\varepsilon\le\frac34.
\]
Any estimator that succeeds at both $k$ and $k'$ yields a test:
decide $k'$ if its output is at least $k'$, and decide $k$ otherwise.
Both testing errors are at most $\delta$. Since $\delta\le1/8$ gives
$\log(1/(4\delta))\ge\frac13\log(1/\delta)$, the preceding bound
implies
\[
m\ge\frac{32}{243\varepsilon^2}\log\!\left(\frac1\delta\right).
\]

To see that this observation model occurs within our clusterability
regime, take $d,n\ge k'$ and, under either hypothesis $h\in\{k,k'\}$,
form a length-$n$ stream containing each of the orthogonal unit vectors
$e_1,\ldots,e_h$ at least once, padding with exact copies of $e_1$.
For independent $w_j\sim\mathcal N(0,I_d)$, the $j$th register is
\[
\max_{x\in X_h}\langle w_j,x\rangle
=\max_{1\le r\le h}(w_j)_r
\sim P_h,
\]
independently across $j$. These streams have zero center correlation
and zero within-cluster perturbation. The lower bound concerns the
resulting register values alone; it does not give the estimator access
to the vectors $w_j$.
\end{proof}

\section{Experimental Details}

\begin{figure}[!hbt]
    \centering
    \includegraphics[width=0.7\linewidth]{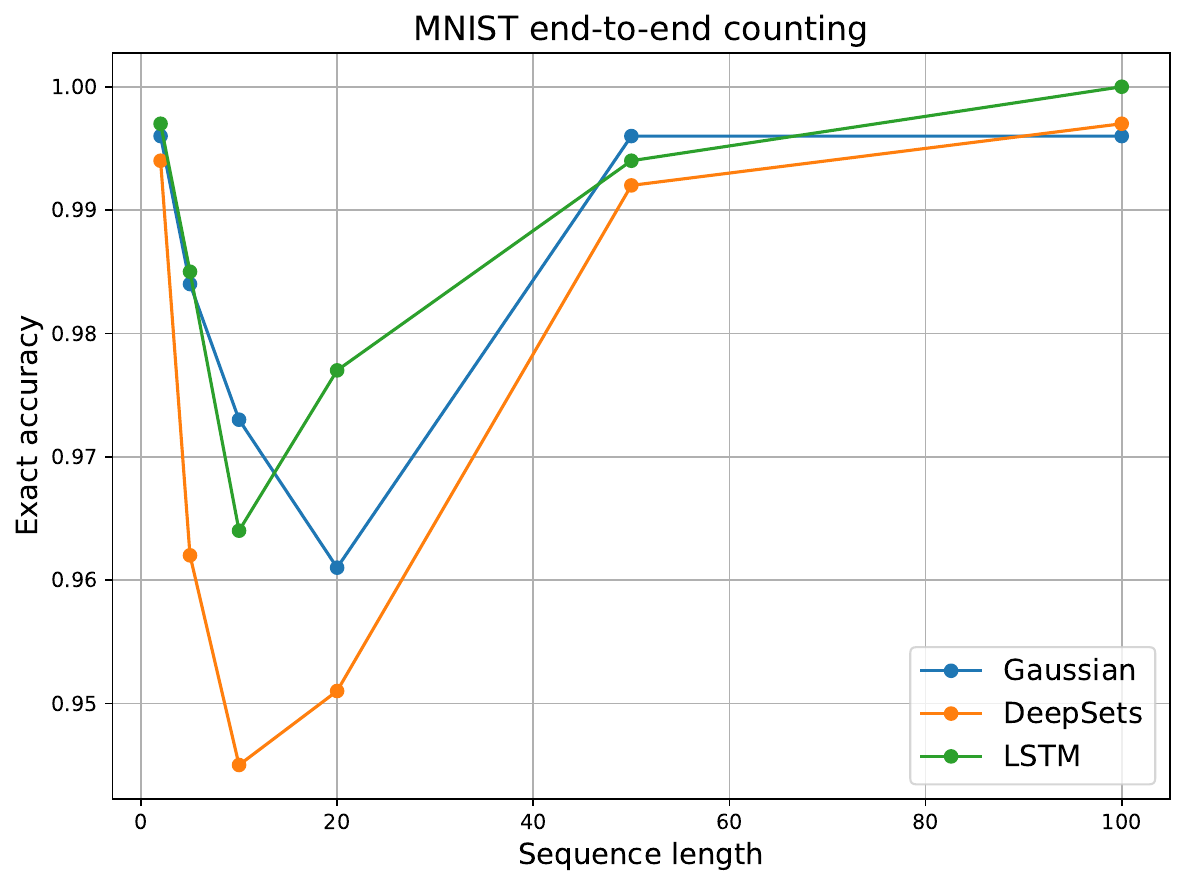}
    \caption{Exact counting accuracy of three baseline models, MaxSketch (Gaussian), DeepSets, and LSTM, on MNIST sequences of varying lengths. All models achieve near-perfect accuracy across all sequence lengths, and the $\pm 1$ accuracy reaches 100\%, indicating that predictions are at most one count away from the true number of digits.}
    \label{fig:mnist_count_generalization}
\end{figure}

\subsection{MNIST}

\paragraph{Dataset and preprocessing.}
We use the standard MNIST dataset consisting of 60{,}000 training images and 10{,}000 test images. All images are scaled to the range $[0,1]$ and standardized using the training set mean and standard deviation. No data augmentation is applied.

\paragraph{Encoder.}
We use a small convolutional neural network composed of two convolutional blocks. The first block applies a $3\times3$ convolution with 32 channels, followed by batch normalization, a ReLU activation, and $2\times2$ max pooling. The second block applies a $3\times3$ convolution with 64 channels, again followed by batch normalization, ReLU, and $2\times2$ max pooling. The resulting feature maps are flattened and passed through a linear layer producing an embedding of dimension $d = 256$. All embeddings are L2-normalized before being passed to the aggregation module.

\paragraph{Aggregators.}
We evaluate three aggregation mechanisms. MaxSketch uses $m = 4{,}096$ fixed random Gaussian projections and computes the maximum projection along each direction. These statistics are then passed through a two-layer MLP with hidden width $256$ and ReLU activation to produce the final count prediction. DeepSets applies a two-layer MLP with hidden width $256$ and output width $256$ with ReLU activation to each embedding individually, aggregates representations via summation, and then applies a final two-layer MLP with hidden width $256$ and ReLU activation to predict the count. The LSTM baseline processes the sequence of embeddings using a single-layer LSTM with hidden size $256$, and the final hidden state is mapped to the count via a linear layer.

\paragraph{Training.}
All models are trained using the Adam optimizer with learning rate $0.001$ and batch size $32$ sequences. The loss function is mean squared error between predicted and true counts. Training is performed for $50$ epochs, with evaluation after each epoch.

\paragraph{Stream construction.}
Each training sequence is generated by first sampling $k \sim \mathrm{Uniform}\{1,\dots,10\}$. We then sample $k$ classes uniformly without replacement, and construct a sequence of length $n$ by repeatedly sampling a class uniformly from the chosen set and then sampling an image uniformly from that class. Each sequence is resampled independently at every epoch. Test sequences are constructed using the same procedure but with test images.

\paragraph{Experiments.}
We evaluate accuracy across sequence lengths by training and testing at each $n \in \{2, 5, 10, 20, 50, 100\}$ independently. For extrapolation, MaxSketch is trained at $n_{\text{train}} \in \{10, 20, 50\}$ and tested at $n_{\text{test}} \in \{150, 250, 500\}$. In addition, we compute 1-nearest-neighbour accuracy using cosine distance to the nearest training embedding.

\subsection{CIFAR-10}
\label{app:cifar10}
 
\paragraph{Dataset.} Standard CIFAR-10 split (50,000 training, 10,000 test
images). We hold out a random 10\% of the training images for validation,
which is used only to choose the number of principal directions. No data
augmentation is used.
 
\paragraph{Encoders.} The publicly released DINOv2 ViT-L/14 and ViT-B/14
models (\texttt{timm}: \texttt{vit\_large\_patch14\_dinov2.lvd142m},
\texttt{vit\_base\_patch14\_dinov2.lvd142m}), trained by self-distillation
without labels. Images are resized to $224\times224$ (bicubic) and normalised
with the ImageNet statistics; the embedding is the class-token output. The
encoder is frozen and never sees CIFAR-10 labels.
 
\paragraph{Label-free preprocessing.} On the unlabeled training embeddings we
compute the mean and the top principal directions. Each embedding is centred,
projected onto the top $q$ directions and L2-normalised; test embeddings use
the same map. We use $q=64$ for DINOv2-L with the MLP readout and $q=16$ for
DINOv2-B and for the isotonic readout, chosen on validation streams from
$q\in\{16,64\}$ (MLP) and $\{16,64,128\}$ plus raw and centred-only
embeddings (isotonic).
 
\paragraph{MaxSketch.} $m=4{,}096$ fixed Gaussian projections drawn from
$\mathcal N(0,I_q)$, redrawn for each seed.
 
\paragraph{MLP readout.} LayerNorm on the $m$ registers, followed by a
two-layer MLP (hidden width 512, GELU, dropout 0.1) with a softmax over
$k\in\{1,\dots,10\}$, trained with cross-entropy on count labels only. We use
AdamW (learning rate $10^{-3}$, weight decay $10^{-4}$) with a one-cycle
schedule for 8,000 steps. Each step draws a length
$n\in\{2,3,5,10,20,50,100\}$ and a batch of $512$
training-image streams. A single model is used for every test length.
 
\paragraph{Isotonic readout.} For each test length $n$, we fit
\texttt{scikit-learn}'s \texttt{IsotonicRegression} from $S(X)$ to $k$ on
4,000 training-image streams of length $n$ and round its output to the
nearest integer in $\{1,\dots,10\}$. The calibration depends on $n$, as do
the thresholds of the estimator in Theorem~\ref{thm:main}.
 
\paragraph{Stream construction.} We sample
$k\sim\mathrm{Uniform}\{1,\dots,10\}$, sample $k$ classes without
replacement, and draw $n$ images by repeatedly picking one of the $k$ classes
uniformly and then an image uniformly from that class; the target is the
number of distinct classes in the stream. Training and calibration streams use
training images, and test streams use test images. We evaluate 2,000 test
streams per length $n\in\{10,20,50,100,250,500\}$ and report the mean over
three seeds (projections, streams and readout initialisation).

\subsection{Faces (CelebA)}
\paragraph{Dataset.} We use CelebA \cite{liu2015deep} (202{,}599 images,
10{,}177 identities). We keep identities with at least two images and
remove identities whose mean embeddings have cosine similarity above
$0.85$ with another identity's, which likely denote the same person under
two labels. The remaining $9{,}868$ identities ($197{,}063$ images) are
split uniformly at random by identity into a calibration set and an
evaluation set of $4{,}934$ identities each.

\paragraph{Embeddings.} Faces are detected and aligned with MTCNN
\cite{zhang2016joint}; for the $310$ images with no detected face we use a
centre crop. Each $160\times160$ face crop is embedded with an
Inception-ResNet-v1 FaceNet encoder \cite{schroff2015facenet} pretrained
on VGGFace2 \cite{cao2018vggface2} (the \texttt{facenet-pytorch}
implementation), giving $512$-dimensional embeddings. Embeddings are
centred by the mean of the calibration identities and L2-normalized. The
encoder is frozen throughout.

\paragraph{Streams.} A stream with $k$ identities samples $k$ identities
without replacement, includes each at least once, and fills the remaining
positions of a length-$n=20{,}000$ stream by sampling an identity uniformly
and then one of its images uniformly. Evaluation uses
$k\in\{100,200,\dots,4900\}$ with $100$ streams per $k$. Calibration
streams use $200$ geometrically spaced values of $k$ between $80$ and
$4{,}934$, with $10$ streams each.

\paragraph{MaxSketch.} $m=4{,}096$ projections drawn from
$\mathcal N(0,I_{512})$. The readout is an isotonic regression of $\log k$
on $S(X)$ (scikit-learn), fit on the calibration streams. Results are
averaged over three independent draws of the projections.

\paragraph{Zhang baseline.} We implement Algorithm 1 of
\cite{zhang2025robust} with Euclidean distance on the normalized
embeddings and $\varepsilon=0.1$. With the paper's constants and
$n=20{,}000$ the algorithm never subsamples, so it keeps one representative
per group. The threshold $\alpha\in\{0.6,0.7,\dots,1.4\}$ is chosen to
minimize mean relative error on $30$ calibration streams ($\alpha=1.1$),
and the algorithm is evaluated on $5$ streams per $k$.

%%%%%%%%%%%%%%%%%%%%%%%%%%%%%%%%%%%%%%%%%%%%%%%%%%%%%%%%%%%%

\newpage
\section*{NeurIPS Paper Checklist}

\begin{enumerate}

\item {\bf Claims}
    \item[] Question: Do the main claims made in the abstract and introduction accurately reflect the paper's contributions and scope?
    \item[] Answer:\answerYes{} % Replace by \answerYes{}, \answerNo{}, or \answerNA{}.
    \item[] Justification: The abstract and introduction accurately reflect the paper’s theoretical contributions, assumptions, and scope. In particular, they clearly state the clusterability assumptions, the Gaussian sketching framework, and the multiplicative approximation guarantees proved in the paper.

\item {\bf Limitations}
    \item[] Question: Does the paper discuss the limitations of the work performed by the authors?
    \item[] Answer: \answerYes{} % Replace by \answerYes{}, \answerNo{}, or \answerNA{}.
    \item[] Justification: The paper discusses the main limitations and assumptions of the approach, including the $(\eta,\rho)$-clusterability assumption, and what happens when it is violated. It also covers the gap between the theory, which uses fixed projections and a monotone readout of $S(X)$, and the MLP readout used in the experiments.

\item {\bf Theory assumptions and proofs}
    \item[] Question: For each theoretical result, does the paper provide the full set of assumptions and a complete (and correct) proof?
    \item[] Answer: \answerYes{} % Replace by \answerYes{}, \answerNo{}, or \answerNA{}.
    \item[] Justification: All lemmas, corollaries, and theorems are stated with explicit assumptions, and complete proofs are provided in the appendix.

    \item {\bf Experimental result reproducibility}
    \item[] Question: Does the paper fully disclose all the information needed to reproduce the main experimental results of the paper to the extent that it affects the main claims and/or conclusions of the paper (regardless of whether the code and data are provided or not)?
    \item[] Answer: \answerYes{} % Replace by \answerYes{}, \answerNo{}, or \answerNA{}.
    \item[] Justification: The paper provides sufficient implementation and experimental details, including dataset construction, model architecture, training setup, and evaluation protocol, to reproduce the main empirical results.

\item {\bf Open access to data and code}
    \item[] Question: Does the paper provide open access to the data and code, with sufficient instructions to faithfully reproduce the main experimental results, as described in supplemental material?
    \item[] Answer: \answerNo{} % Replace by \answerYes{}, \answerNo{}, or \answerNA{}.
    \item[] Justification: The main contribution of the paper is theoretical. While code and data are not currently released, the paper provides sufficient algorithmic, mathematical, and experimental details to reproduce the proposed method and empirical evaluations.

\item {\bf Experimental setting/details}
    \item[] Question: Does the paper specify all the training and test details (e.g., data splits, hyperparameters, how they were chosen, type of optimizer) necessary to understand the results?
    \item[] Answer: \answerYes{} % Replace by \answerYes{}, \answerNo{}, or \answerNA{}.
    \item[] Justification: The paper specifies the full experimental setup, including data construction, model architecture, hyperparameters, optimization procedure, and evaluation protocol, with additional implementation details provided in the appendix sufficient to understand and interpret the reported results.

\item {\bf Experiment statistical significance}
    \item[] Question: Does the paper report error bars suitably and correctly defined or other appropriate information about the statistical significance of the experiments?
    \item[] Answer: \answerYes{} % Replace by \answerYes{}, \answerNo{}, or \answerNA{}.
    \item[] Justification: The experiments are reported based on multiple runs of the proposed method over random draws of Gaussian projections and dataset samples.

\item {\bf Experiments compute resources}
    \item[] Question: For each experiment, does the paper provide sufficient information on the computer resources (type of compute workers, memory, time of execution) needed to reproduce the experiments?
    \item[] Answer: \answerNo{} % Replace by \answerYes{}, \answerNo{}, or \answerNA{}.
    \item[] Justification: The paper does not provide detailed information on compute resources such as hardware specifications, memory usage, or exact runtime per experiment, as the experiments are lightweight and primarily involve sampling Gaussian projections and computing simple statistics.
    
\item {\bf Code of ethics}
    \item[] Question: Does the research conducted in the paper conform, in every respect, with the NeurIPS Code of Ethics \url{https://neurips.cc/public/EthicsGuidelines}?
    \item[] Answer: \answerYes{} % Replace by \answerYes{}, \answerNo{}, or \answerNA{}.
    \item[] Justification: The research uses only synthetic or standard benchmark-style data and mathematical modeling, does not involve human subjects or personal data, and does not raise known ethical concerns related to privacy, safety, discrimination, or misuse, thus conforming to the NeurIPS Code of Ethics.

\item {\bf Broader impacts}
    \item[] Question: Does the paper discuss both potential positive societal impacts and negative societal impacts of the work performed?
    \item[] Answer: \answerNA{} % Replace by \answerYes{}, \answerNo{}, or \answerNA{}.
    \item[] Justification: The work is primarily theoretical in nature, focusing on Gaussian sketching and statistical guarantees for cluster estimation, and does not involve deployment-oriented systems or direct applications; therefore, broader societal impacts are minimal and not explicitly discussed.
    
\item {\bf Safeguards}
    \item[] Question: Does the paper describe safeguards that have been put in place for responsible release of data or models that have a high risk for misuse (e.g., pre-trained language models, image generators, or scraped datasets)?
    \item[] Answer: \answerNA{} % Replace by \answerYes{}, \answerNo{}, or \answerNA{}.
    \item[] Justification: No safeguards are described, as the paper does not release high-risk data or models and is primarily theoretical in nature.

\item {\bf Licenses for existing assets}
    \item[] Question: Are the creators or original owners of assets (e.g., code, data, models), used in the paper, properly credited and are the license and terms of use explicitly mentioned and properly respected?
    \item[] Answer: \answerYes{} % Replace by \answerYes{}, \answerNo{}, or \answerNA{}.
    \item[] Justification: All datasets and models are cited in the references.

\item {\bf New assets}
    \item[] Question: Are new assets introduced in the paper well documented and is the documentation provided alongside the assets?
    \item[] Answer: \answerNA{} % Replace by \answerYes{}, \answerNo{}, or \answerNA{}.
    \item[] Justification: No new assets are released.

\item {\bf Crowdsourcing and research with human subjects}
    \item[] Question: For crowdsourcing experiments and research with human subjects, does the paper include the full text of instructions given to participants and screenshots, if applicable, as well as details about compensation (if any)? 
    \item[] Answer: \answerNA{} % Replace by \answerYes{}, \answerNo{}, or \answerNA{}.
    \item[] Justification: The paper does not involve crowdsourcing nor research with human subjects.

\item {\bf Institutional review board (IRB) approvals or equivalent for research with human subjects}
    \item[] Question: Does the paper describe potential risks incurred by study participants, whether such risks were disclosed to the subjects, and whether Institutional Review Board (IRB) approvals (or an equivalent approval/review based on the requirements of your country or institution) were obtained?
    \item[] Answer: \answerNA{} % Replace by \answerYes{}, \answerNo{}, or \answerNA{}.
    \item[] Justification: Paper does not involve crowdsourcing nor research with human subjects.

\item {\bf Declaration of LLM usage}
    \item[] Question: Does the paper describe the usage of LLMs if it is an important, original, or non-standard component of the core methods in this research? Note that if the LLM is used only for writing, editing, or formatting purposes and does \emph{not} impact the core methodology, scientific rigor, or originality of the research, declaration is not required.
    %this research? 
    \item[] Answer: \answerNA{} % Replace by \answerYes{}, \answerNo{}, or \answerNA{}.
    \item[] Justification: The core method development in this research does not involve LLMs as any important, original, or non-standard components.

\end{enumerate}

\end{document}